\documentclass[11pt]{amsart}
\usepackage{hyperref}
\usepackage{halloweenmath}
\usepackage{amsmath,amsfonts,amssymb,amscd,amsthm,amsbsy,amsxtra,bbm,bm, epsf,calc,comment,appendix, xcolor}
\usepackage{color}
\usepackage{datetime}
\usepackage{latexsym}
\usepackage[english]{babel}
\usepackage{enumerate}
\usepackage{graphicx}
\usepackage{epsfig}
\usepackage{dsfont}
\usepackage{tikz}

\definecolor{mygreen}{rgb}{0.1,0.75,0.2}

\DeclareMathOperator{\Prob}{\mathbf{P}}

\DeclareSymbolFont{bbold}{U}{bbold}{m}{n}
\DeclareSymbolFontAlphabet{\mathbbold}{bbold}

\newcommand{\spt}{\textup{spt}}

\newcommand{\pretv}{\widetilde{\text{TV}}}

\newcommand{\X}{\mathcal{X}}
\newcommand{\Y}{\mathcal{Y}}
\newcommand{\Z}{\mathcal{Z}}
\newcommand{\F}{\mathcal{F}}
\newcommand{\R}{\mathbb{R}}
\def\N{\mathbb{N}}

\numberwithin{equation}{section}
\newtheorem{theorem}{Theorem}[section]
\newtheorem{lemma}[theorem]{Lemma}
\newtheorem{proposition}[theorem]{Proposition}
\newtheorem{corollary}[theorem]{Corollary}

\newtheorem{assumption}[theorem]{Assumption}

\theoremstyle{remark}
\newtheorem{remark}[theorem]{Remark}

\theoremstyle{definition}
\newtheorem{definition}[theorem]{Definition}

\usepackage{amsthm}
\usepackage{hyperref}
\hypersetup
{colorlinks = true,
linkcolor = red, 
anchorcolor = red, 
citecolor = blue, 
filecolor = blue,
urlcolor = blue}
\usepackage{hypcap}
\usepackage{cleveref}
\usepackage{lipsum}

\usepackage{enumerate}
\usepackage{enumitem}

\title[Existence of Robust classifier]{On the existence of solutions to adversarial training in multiclass classification}

\author{Nicol\'as {Garc\'ia Trillos}}
\address{Department of Statistics, University of Wisconsin-Madison, 1300 University Avenue, Madison, Wisconsin 53706, USA.}
    \email{garciatrillo@wisc.edu}
\author{Matt Jacobs}
\address{Department of Mathematics, 
Purdue University, 
Mathematical Sciences Bldg, 150 N University St, West Lafayette, IN 47907}
\email{jacob225@purdue.edu}
\author{Jakwang Kim}
\address{Department of Statistics, University of Wisconsin-Madison, 1300 University Avenue, Madison, Wisconsin 53706, USA.}
\email{kim836@wisc.edu}

\date{\today}


\begin{document}

\maketitle

\begin{abstract}
We study three models of the problem of adversarial training in multiclass classification designed to construct robust classifiers against adversarial perturbations of data in the agnostic-classifier setting. We prove the existence of Borel measurable robust classifiers in each model and provide a unified perspective of the adversarial training problem, expanding the connections with optimal transport initiated by the authors in their previous work \cite{trillos2023multimarginal} and developing new connections between adversarial training in the multiclass setting and total variation regularization.  As a corollary of our results, we prove the existence of Borel measurable solutions to the agnostic adversarial training problem in the binary classification setting, a result that improves results in the literature of adversarial training, where robust classifiers were only known to exist within the enlarged universal $\sigma$-algebra of the feature space. 
\end{abstract}

\section{Introduction}

Modern machine learning models, in particular those generated with deep learning, perform remarkably well, in many cases much better than humans, at classifying data in a variety of challenging application fields like image recognition, medical image reconstruction, and natural language processing. However, the robustness of these learning models to data perturbations is a completely different story.  For example, in image recognition, it has been widely documented (e.g., \cite{goodfellow_examples}) that certain structured but human-imperceptible modifications of images at the pixel level can fool an otherwise well-performing image classification model.  These small data perturbations, known as \emph{adversarial attacks}, when deployed at scale can make a  model's prediction accuracy drop substantially and in many cases collapse altogether. As such, they are a significant obstacle to the deployment of machine learning systems in security-critical applications, e.g. \cite{BIGGIO2018317}. To defend against these attacks, many researchers have investigated the problem of adversarial training, i.e., training methods that produce models that are robust to attacks.  In adversarial training, one typically pits the adversary against the learner during the training step, forcing the learner to select a model that is robust against attacks. Nonetheless, despite the attention that has been devoted to understanding these problems, theoretically and algorithmically, there are still several important mathematical questions surrounding them that have not been well understood.

A fundamental difficulty in adversarial training, in contrast to standard training of learning models, is the fact that the adversary has the power to alter the underlying data distribution. In particular, model training becomes an implicit optimization problem over a space of measures, as a result, one may be forced to leave the prototypical setting of equivalence classes of functions defined over a single fixed measure space. In general, measurability issues become more delicate for adversarial training  problems at the moment of providing a rigorous mathematical formulation for the 
problem.
Due to these difficulties, there are several subtle variations of the adversarial training model in the literature and it has not been clear whether these models are fully equivalent. More worryingly, for some models, even the \emph{existence} of optimal robust classifiers is unknown, essentially due to convexity and compactness issues.

Let us emphasize that these issues arise even in what can be regarded as the simplest possible setting of the agnostic learner, i.e. where the space of classifiers is taken to be the set of all possible Borel measurable weak (probabilistic) classifiers.  While this setting is trivial in the absence of an adversary (there the optimal choice for the learner is always the Bayes classifier),  the structure of the problem is much more subtle in the adversarial setting (in other words the analog of the Bayes classifier is not fully understood).  With an adversary, the training process can be viewed as a two-player min-max game (learner versus adversary) \cite{bose2020adversarial, Meunier2021MixedNE, pydi2021the, pmlr-v206-balcan23a} and as a result, the optimal strategies for the two players are far from obvious.  By relaxing the problem to the agnostic setting, one at least is working over a convex space, but again measurability issues pose a problem  for certain formulations of adversarial training.

In light of the above considerations, the purpose of this paper is twofold. On one hand, we provide rigorous justification for the existence of \textit{Borel}-measurable robust classifiers in the multiclass classification setting for three different models of adversarial training.  Notably, our analysis includes a widely used model for which the existence of Borel classifiers was not previously known and existence of solutions had only been guaranteed when enlarging the original Borel $\sigma$-algebra of the data space. On the other hand, we develop a series of connections between the three mathematical models of adversarial training discussed throughout the paper exploiting ideas from optimal transportation and total variation minimization. By developing these connections, we hope to present a unified formulation of adversarial training and highlight the prospective advantages of using tools in computational optimal transport for solving these problems in the agnostic-classifier setting (and perhaps beyond the agnostic setting too). We also highlight, in concrete terms, the connection between adversarial training and the direct regularization of learning models. To achieve all the aforementioned goals, we expand and take advantage of our previous work \cite{trillos2023multimarginal} as well as of the work \cite{bungert2023geometry} exploring the connection between adversarial training and perimeter minimization in the binary classification setting.



\subsection{Organization of the paper}
The rest of the paper is organized as follows. In section \ref{preminaries}, we introduce three different models for adversarial training in the multiclass classification setting that we will refer to as the open-ball model, the closed-ball model, and the distributional-perturbing (DRO) model. In section \ref{sec:MainResults} we state our main mathematical results and in section \ref{sec:Discussion} we discuss related literature and some of the implications of our results. In section \ref{section : existence dual potential} we lay down the main mathematical tools for analyzing the DRO model. Part of these tools come directly from our previous work \cite{trillos2023multimarginal}, while others are newly developed. In section \ref{section : Borel robust classifier} we prove our main results: first, we prove the existence of solutions for the DRO model (section \ref{sec:WellPosednes}); then we prove that solutions to the DRO model are solutions to the closed-ball model (section \ref{sec:ClosedBallWell}); finally, we relate the closed-ball model to the open-ball model in section \ref{section: unify all models}. Lastly, in section \ref{section : conclusion and future works} we wrap up the paper and discuss future research directions.

\section{Set-up and main results}\label{preminaries}
The setting of our problem will be a feature space $(\mathcal{X},d)$ (a Polish space with metric $d$) and a label space $\Y:=\{1, \dots, K\}$, which will represent a set of $K$ labels for a given classification problem of interest. We denote by $\mathcal{Z}:=\mathcal{X} \times \Y$ the set of input-to-output pairs and endow it with a Borel probability measure $\mu \in \mathcal{P}(\Z)$,  representing a ground-truth data distribution. For convenience, we will often describe the measure $\mu$ in terms of its class probabilities $\mu=(\mu_1, \dots, \mu_K)$, where each $\mu_i$ is the positive Borel measure (not necessarily a probability measure) over $\mathcal{X}$ defined according to:
\[  \mu_i(A) = \mu (A \times \{ i\}),\]
for $A \in\mathfrak{B}(\X)$, i.e., $A$ is a Borel-measurable subset of $\X$. Notice that the measures $\mu_i$ are, up to normalization factors, the conditional distributions of inputs/features given the different output labels.

Typically, a (multiclass) classification rule in the above setting is simply a Borel measurable map $f: \X \rightarrow \Y$. In this paper, however, it will be convenient to expand this notion slightly and interpret general classification rules as Borel measurable maps from $\X$ into  $\Delta_{\Y} := \left\{ (u_i)_{i \in \Y} : 0 \leq u_i \leq 1, \sum_{i \in \Y} u_i \leq 1 \right\}$, the set of (up to normalization constants) probability distributions over $\Y$ (see remark \ref{rem:01Loss}); oftentimes these functions are known as \textit{soft-classifiers}. For future reference, we denote by $\F$ the set
\begin{equation}
    \mathcal{F}:= \left\{ f: \mathcal{X} \to \Delta_{\Y} : f \text{ is Borel measurable} \right\}.
    \label{def:F}
\end{equation}
Given $f \in \F$ and $x \in \X$, the vector $f(x) = (f_1(x), \dots, f_K(x))$ will be interpreted as the vector of probabilities over the label set $\Y$ that the classifier $f$ assigns to the input data point $x$. In practice, from one such $f$ one can induce actual (hard) class assignments to the different inputs $x$ by selecting the coordinate in $f(x)$ with largest entry. The extended notion of classifier considered in this paper is actually routinely used in practice as it fares well with the use of standard optimization techniques (in particular, $\F$ is natural as it can be viewed as a convex relaxation of the space of maps from $\X$ to $\Y$).





The goal in the standard (unrobust) classification problem is to find a classifier $f\in \F$ that gives accurate class assignments to inputs under the assumption that data points are distributed according to the ground-truth distribution $\mu$. This aim can be mathematically modeled as an optimization problem of the form:
\begin{equation}
    \inf_{f \in \mathcal{F}} R(f, \mu),
    \label{eq:UnrobustRiskMinimization}
\end{equation}
where $R(f, \mu)$ is the risk of a classifier $f$ relative to the data distribution $\mu$:
\begin{equation*}
    R(f, \mu):=\mathbb{E}_{(X,Y) \sim \mu} [\ell (f(X), Y)) ].
\end{equation*}
The loss function $\ell: \Delta_\Y \times \Y  \rightarrow \R$ appearing in the definition of the risk can be chosen in multiple reasonable ways, but here we restrict to the choice 
 \[\ell(u, i) := 1 - u_i, \quad (u,i) \in \Delta_\Y \times \Y,\] 
 which, in lieu of the fact that $\ell(e_j , i) $ is equal to $1$ if $i \not = j$ and $0$ if $i=j$ ($e_j$ is the extremal point of $\Delta_\Y$ with entry one in its $j$-th coordinate), will be referred to as the $0$-$1$ loss. Note that under the $0$-$1$ loss function the risk $R(f, \mu)$ can be rewritten as 
 \[ R(f, \mu)= \sum_{i \in \Y} \int_{\mathcal{X}} \left(1 - f_i(x) \right) d\mu_i(x).\]
Moreover, one can observe that solutions to the risk minimization problem \eqref{eq:UnrobustRiskMinimization} are the standard multiclass Bayes classifiers from statistical learning theory (e.g., see \cite{Bousquet2004,LUXBURG2011651}). These classifiers are characterized by the condition $f_{Bayes,i}^*(x)=0$ if $\Prob_{(X, Y)\sim \mu}(Y=i| X=x) \not = \max_{j \in \Y} \Prob_{(X, Y)\sim \mu}(Y=j| X=x) $ for all $i$, and it is always possible to select a Bayes classifier of the form  $f^*_{Bayes}(x) = (\mathds{1}_{A^*_1}(x), \dots, \mathds{1}_{A^*_K}(x))$, where $A_1^*, \dots, A_K^*$ form a measurable partition of $\X$. In other words, there always exist hard classifiers that solve the risk minimization problem \eqref{eq:UnrobustRiskMinimization}.

By definition, a solution to \eqref{eq:UnrobustRiskMinimization} classifies \textit{clean} data optimally; by clean data here we mean data distributed according to the original distribution $\mu$. However, one should not expect the standard Bayes classifier to perform equally well when inputs have been adversarially contaminated, and the goal in adversarial training is precisely to create classifiers that are less susceptible to data corruption. One possible way to enforce this type of robustness is to replace the objective function in \eqref{eq:UnrobustRiskMinimization} with one that incorporates the actions of a well-defined adversary, and then search for the classifier that minimizes the new notion of (adversarial) risk. This adversarial risk can be defined in multiple ways, but two general ways stand out in the literature and will be the emphasis of our discussion; we will refer to these two alternatives as  \textit{data-perturbing adversarial model} and \textit{distribution-perturbing adversarial model}. As it turns out, there exist connections between the two (see \cite{pydi2021the} for more details) and we will develop further connections shortly.

For the \textit{data-perturbing adversarial model} we will consider the following two versions:
\begin{align}
\label{def:open-ball_model}
    &R^o_{\varepsilon} := \inf_{f \in \mathcal{F}} R^o_{\varepsilon}(f) := \inf_{f \in \mathcal{F}} \left\{ \sum_{i \in \Y} \int_{\mathcal{X}} \sup_{\widetilde{x} \in B_{\varepsilon}(x)} \{ 1 - f_i(\widetilde{x}) \} d\mu_i(x) \right\},\\
    &\label{def:closed-ball_model}
    \inf_{f \in \mathcal{F}} \left\{ \sum_{i \in \Y} \int_{\mathcal{X}} \sup_{\widetilde{x} \in \overline{B}_{\varepsilon}(x)} \{ 1 - f_i(\widetilde{x}) \} d\mu_i(x) \right\}.
\end{align}
Here $B_{\varepsilon}(x)$  ($\overline{B}_{\varepsilon}(x)$, respectively) denotes an open (closed) ball with radius $\varepsilon$ centered at $x$. In both versions, the adversary can substitute any given input $x$ with a $\tilde x$ that belongs to a small ball of radius $\varepsilon$ around the original $x$. In this setting, the learner's goal is to minimize the worst-loss that the adversary may induce by carrying out one of their feasible actions. Although at the heuristic level the difference between the two models is subtle (in the first model the adversary optimizes over open balls and in the second over closed balls), at the mathematical level these two models can be quite different. For starters, the problem \eqref{def:closed-ball_model} is not well-formulated, as it follows from a classical result in \cite{luzin1919quelques}, which discusses that, in general, the function $x \mapsto \sup_{\tilde x \in \overline{B}_\varepsilon(x)} \{ 1- f_i(\tilde x) \}$ may not be Borel-measurable when only the Borel-measurability of the function $f_i$ has been assumed. For this reason, the integral with respect to $\mu_i$ in \eqref{def:closed-ball_model} (which is a Borel positive measure, i.e., it is only defined over the Borel $\sigma$-algebra) may not be defined for all $f\in \F$. In \eqref{def:closed-ball_model-corrected} we provide a rigorous formulation of \eqref{def:closed-ball_model} (which at this stage should only be interpreted informally). This reformulation will require the use of an extension of the Borel $\sigma$-algebra, known as the universal $\sigma$-algebra, as well as an extension of the measures $\mu_i$ to this enlarged $\sigma$-algebra. Problem \eqref{def:open-ball_model}, on the other hand, is already well formulated, as no measurability issues arise when taking the sup over open balls. At a high level, this is a consequence of the fact that arbitrary unions of open balls are open sets and thus Borel-measurable; see, for example, Remark 2.3 in \cite{bungert2023geometry}.  Regardless of which of the two models one adopts, and putting aside for a moment the measurability issues mentioned above, it is unclear whether it is possible to find minimizers for any of the problems  \eqref{def:open-ball_model} and \eqref{def:closed-ball_model} within the family $\F$.

The \textit{distributional-perturbing adversarial model} is defined as a minimax problem that can be described as follows: after the learner has chosen a classifier $f \in \F$, an adversary selects a new data distribution $\widetilde{\mu} \in \mathcal{P}(\mathcal{Z})$, and, by paying some cost $C(\mu, \widetilde{\mu})$, attempts to make the risk $R(f, \widetilde \mu)$ be as large as possible. Precisely, we consider the problem
\begin{equation}\label{def:DRO_model}
    R_{DRO}^*:=\inf_{f \in \mathcal{F}} \sup_{\widetilde{\mu} \in \mathcal{P}(\mathcal{Z})} \left\{ R(f, \widetilde{\mu}) - C(\mu, \widetilde{\mu}) \right\},
\end{equation}
where $C : \mathcal{P}(\mathcal{Z}) \times \mathcal{P}(\mathcal{Z}) \to [0, \infty ]$ has the form:
\begin{equation*}
    C(\mu, \widetilde{\mu}) := \inf_{\pi \in \Gamma(\mu, \widetilde{\mu})} \int c_{\mathcal{Z}}(z, \widetilde{z}) d\pi(z, \widetilde{z}),
\end{equation*}
for some Borel measurable cost function $ c_{\mathcal{Z}}: \mathcal{Z} \times \mathcal{Z} \rightarrow [0,\infty]$. Here and in the remainder of the paper, we use $\Gamma(\cdot, \cdot)$ to represent the set of couplings between two positive measures over the same space; for example, $\Gamma(\mu, \widetilde{\mu})$ denotes the set of positive measures over $\mathcal{Z} \times \mathcal{Z}$ whose first and second marginals are ${\mu}$ and $\widetilde \mu$, respectively. Note that problem \eqref{def:DRO_model} is an instance of a \textit{distributionally robust optimization} (DRO) problem. Problem \eqref{def:DRO_model} is well-defined given that all its terms are written as integrals of Borel measurable integrands against Borel measures.

In the remainder, we will assume that the cost $c_\Z: \Z \times \Z \rightarrow [0,\infty]$ has the form
\begin{equation}
\label{assump:CostStructure}
c_\Z(z,\tilde z) := \begin{cases} c(x, \tilde x ) & \text{ if } y=\tilde y \\ \infty & \text{ otherwise }, \end{cases}
\end{equation}
for a lower semi-continuous function $c : \X \times \X \rightarrow [0,\infty]$. Note that when $c_\Z$ has the above structure we can rewrite $C(\mu, \widetilde \mu)$ as
\[ C(\mu, \widetilde \mu) = \sum_{i=1}^K C(\mu_i, \widetilde \mu_i),  \]
where on the right-hand side we slightly abuse notation and use $C(\mu_i, \widetilde \mu_i)$ to represent
\[  C(\mu_i, \widetilde \mu_i):= \inf_{\pi \in \Gamma(\mu_i, \widetilde \mu_i) } \int c(x, \tilde x) d \pi(x, \tilde x). \]
A typical example of a cost $c$ that we will discuss in detail throughout this paper is the cost function:
\begin{equation}
\label{def:CostEpsilon}
    c(x,\tilde x) = c_\varepsilon(x,\tilde x) := \begin{cases}
    \infty & \text{if } d(x,\tilde x) >\varepsilon \\
    0 & \text{if } d(x, \tilde x) \leq \varepsilon,
    \end{cases}
\end{equation}
where in the above $\varepsilon$ is a positive parameter that can be interpreted as \textit{adversarial budget}.

\begin{remark}
Throughout the paper, we use the convention that $C(\mu_i, \widetilde{\mu}_i) = \infty$ whenever the set of couplings $\Gamma(\mu_i, \widetilde{\mu}_i)$ is empty. This is the case when $\mu_i$ and $\widetilde{\mu}_i$ have different total masses.
\end{remark}

\begin{remark}
\label{rem:01Loss}
Given the structure of the 0-1 loss function considered here, in all the adversarial models introduced above we may replace the set $\F$ with the set of those $f \in \F$ for which $\sum_i f_i =1$. Indeed, given $f \in \F$ we can always consider $\tilde f \in \F$ defined according to $\tilde{f}_{i_0}:= f_{i_0} +(1- \sum_{i\in \Y} f_i) $ and $\tilde f_i = f_i $ for $i\not = i_0$ to obtain a value of risk that is no greater than the one of the original $f$.
\end{remark}

\subsection{Main results}
\label{sec:MainResults}

Our first main theorem discusses the existence of (Borel) solutions for problem \eqref{def:DRO_model} under the assumptions on the cost $c: \X \times \X \rightarrow [0, \infty]$ stated below.

\begin{assumption}
\label{assump}
 We assume that the cost $c: \X\times \X \rightarrow[0, \infty]$ is a lower semi-continuous and symmetric function satisfying $c(x,x)=0$ for all $x\in \X$. We also assume the the following compactness property holds:
if $\{ x_n \}_{n \in \mathbb{N}}$ is a bounded sequence in $(\X, d)$ and $\{ x_n' \}_{n \in \mathbb{N}}$ is a sequence satisfying $\sup_{n \in \mathbb{N}} c(x_n, x_n')< \infty$, then $\{(x_n, x_n') \}_{n \in \mathbb{N}}$ is precompact in $\X \times \X$ (endowed with the product topology).
\end{assumption}

\begin{remark}
\label{rem:BoundedSetsCompact}
Notice that Assumption \ref{assump} implicitly requires bounded subsets of $\X$ to be precompact. 
\end{remark}

\begin{theorem}\label{thm : measurable robust classifier}
Suppose that $c: \X \times \X \rightarrow [0,\infty]$ satisfies Assumption \ref{assump}. Then
there exists a (Borel) solution $f^*$ of the DRO model \eqref{def:DRO_model}. Furthermore, there exists $\widetilde \mu ^* \in \mathcal{P}(\Z)$ such that $( f^*, \widetilde \mu^* )$ is a saddle point for \eqref{def:DRO_model}. In other words, the following holds: for any $\widetilde{\mu} \in \mathcal{P}(\mathcal{Z})$ and any $f \in \mathcal{F}$ we have
\begin{equation}
\label{eq: saddle_point_inequality}
    R(f^*, \widetilde{\mu}) - C(\mu, \widetilde{\mu}) \leq R(f^*, \widetilde{\mu}^*) - C(\mu, \widetilde{\mu}^*) \leq R(f, \widetilde{\mu}^*) - C(\mu, \widetilde{\mu}^*).
\end{equation}
\end{theorem}

When the cost function $c$ is regular enough or when $\mu$ is an empirical measure, we can reduce the problem of finding a solution $f^*$ of \eqref{def:DRO_model} to the problem of solving the dual of a \textit{generalized barycenter problem} or the dual of a \textit{multimarginal optimal transport} problem. These connections were first put forward in our earlier work \cite{trillos2023multimarginal} and will be discussed again in section \ref{section : existence dual potential}, concretely in Proposition \ref{prop:OptimalPair}. Unfortunately, when the cost is only lower semi-continuous (e.g., for $c=c_\varepsilon$ as in \eqref{def:CostEpsilon}) and when $\mu$ is an arbitrary Borel probability measure, we can not directly use the content of Proposition \ref{prop:OptimalPair} to guarantee the existence of (Borel) solutions $f^*$. To overcome this issue, we approximate $c$ with a sequence of continuous costs $c_n$ such that the previous theory applies. We then show that in the limit the Borel measurability of the optimal classifier is preserved. At a high level, we can thus reduce finding solutions for the DRO problem \eqref{def:DRO_model} to that of an MOT or a generalized barycenter (or sequences thereof).

\begin{remark}
When the cost $c$ has the form $c_\varepsilon$ in \eqref{def:CostEpsilon}, Assumption \ref{assump} reduces to the requirement that bounded subsets in $\X$ are precompact, which we are anyway assuming in Assumption \ref{assump}, according to remark \ref{rem:BoundedSetsCompact}. This is the case for Euclidean space or for a smooth manifold of finite dimension endowed with its geodesic distance.
\end{remark}

In order to discuss the existence of solutions to the problem \eqref{def:closed-ball_model} we actually first need to modify the problem and define it properly. To do this, we first introduce the \textit{universal} $\sigma$-algebra of the space $\X$. 

\begin{definition}[Definition 2.2 in \cite{Nishiura}]
Let $\mathcal{B}(\X)$ be the Borel $\sigma$-algebra over $\X$ and let $\mathcal{M}(\X)$ be the set of all signed $\sigma$-finite Borel measures over $\X$. For each $\nu \in \mathcal{M}(\mathcal{X})$, let $\mathcal{L}_{\nu}(\mathcal{X})$ be the completion of $\mathcal{B}(\mathcal{X})$ with respect to $\nu$. The universal $\sigma$-algebra of $\X$ is defined as
\begin{equation*}
    \mathcal{U}(\mathcal{X}):= \bigcap_{\nu \in \mathcal{M}(\mathcal{X})} \mathcal{L}_{\nu}(\mathcal{X}).
\end{equation*}

We will use $\overline{\mathcal{P}}(\mathcal{Z})$ to denote the set of probability measures $\gamma$ over $\Z$ for which  $\gamma_i$ is a universal positive measure (i.e., it is defined over $\mathcal{U}(\X)$) for all $i \in \Y$. For a given probability measure $\mu \in \mathcal{P}(\Z)$ we will denote by $\overline{\mu}$ its universal extension, which we will interpret as  
\[ \overline{\mu}(A \times \{ i \}) := \overline{\mu}_i(A), \quad \forall A \in \mathcal{U}(\X), \]
where $\overline{\mu}_i$ is the extension of $\mu_i$ to $\mathcal{U}(\X)$. Finally, we will use $\overline{\mathcal{U}}(\Z)$ to denote the set of all $f=(f_1, \dots, f_K)$ for which each $f_i$ is universally measurable.
\end{definition}
\begin{remark}
    If $(\mathcal{X}, d) = (\mathbb{R}^n, || \cdot ||)$, then $\mathcal{U}(\X)$ is the set of all Lebesgue measurable sets; see Theorem 4.2 in \cite{Nishiura}. So, any Lebesgue-measurable function is universally measurable and vice-versa.
\end{remark}

Having introduced the above notions, we can reformulate problem \eqref{def:closed-ball_model} as:
\begin{equation}\label{def:closed-ball_model-corrected}
    \overline{R}_{\varepsilon} := \inf_{f \in \mathcal{F}} \overline{R}_{\varepsilon}(f) := \inf_{f \in \mathcal{F}} \left\{ \sum_{i \in \Y} \int_{\mathcal{X}} \sup_{\widetilde{x} \in \overline{B}_{\varepsilon}(x)} \{ 1 - f_i(\widetilde{x}) \} d\overline{\mu}_i(x) \right\}.
\end{equation}
Although the difference with \eqref{def:closed-ball_model} is subtle (in \eqref{def:closed-ball_model} we use $\mu_i$ whereas in \eqref{def:closed-ball_model-corrected} we use $\overline{ \mu_i}$), problem \eqref{def:closed-ball_model-corrected} is actually well-defined. Indeed, combining Lemma 4.2 in \cite{pydi2021the} with Corollary 7.42.1 in \cite{bertsekas1996stochastic}, originally from \cite{luzin1919quelques}, it follows that for any Borel measurable $f_i$ the function $x \mapsto \sup_{x \in \overline{B}_\varepsilon(x)} \{1- f_i(\tilde x) \}$ is universally measurable and thus the integrals on the right hand side of \eqref{def:closed-ball_model-corrected} are well defined. 

Our second main result relates solutions of \eqref{def:DRO_model} with solutions of \eqref{def:closed-ball_model-corrected}.

\begin{theorem}\label{prop : borel_general_closed_ball}
There exists a Borel solution of \eqref{def:DRO_model} for the cost function $c= c_\varepsilon$ from \eqref{def:CostEpsilon} that is also a solution of \eqref{def:closed-ball_model-corrected}. In particular, there exists a (Borel) solution for \eqref{def:closed-ball_model-corrected}.  
\end{theorem}

Finally, we connect problem \eqref{def:closed-ball_model-corrected} with problem \eqref{def:open-ball_model}. 

\begin{theorem}
\label{thm:AllmodelsSame}
For all but at most countably many $\varepsilon \geq 0$, we have
$  R_\varepsilon^o = \overline{R}_{\varepsilon}$. Moreover, for those $\varepsilon \geq 0$ for which this equality holds, every solution $f^*$ of \eqref{def:closed-ball_model-corrected} is also a solution of \eqref{def:open-ball_model}.
\end{theorem}

\begin{remark}
In general, we can not expect the optimal adversarial risks of open-ball and closed-ball models to agree for all values of $\varepsilon$. To illustrate this, consider the simple setting of a two class problem (i.e., $K=2$) where $\mu_1 = \frac{1}{2}\delta_{x_1}$ and $\mu_2 = \frac{1}{2}\delta_{x_2}$. Let $\epsilon_0 = \frac{1}{2}d(x_1, x_2)$. It is straightforward to check that $R^o_{\epsilon_0} = 0$ whereas $\overline{R}_{\epsilon_0} =1/2$. Naturally, if we had selected any other value for $\varepsilon>0$ different from $\epsilon_0$ we would have obtained $R^o_{\varepsilon} = \overline{R}_{\varepsilon}$. 
\end{remark}

From Theorem \ref{thm : measurable robust classifier}, Proposition \ref{prop : borel_general_closed_ball}, and Theorem \ref{thm:AllmodelsSame} we may conclude that it is essentially sufficient to solve problem \eqref{def:DRO_model} to find a solution for all other formulations of the adversarial training problem discussed in this paper. Our results thus unify all notions of adversarial robustness into the single DRO problem \eqref{def:DRO_model}. The advantage of \eqref{def:DRO_model} over the other formulations of the adversarial training problem is that it can be closely related to a generalized barycenter problem or an MOT problem, as has been discussed in detail in our previous work \cite{trillos2023multimarginal} (see also section \ref{section : existence dual potential} below). In turn, either of those problems can be solved using computational optimal transport tools. From a practical perspective, it is thus easier to work with the DRO formulation than with the other formulations of adversarial training.

\subsection{Discussion and literature review}
\label{sec:Discussion}

The existence of measurable ``robust" solutions to optimization problems has been a topic of interest not only in the context of adversarial training \cite{pydi2021the, frank2022consistency, frank2022existence, awasthi2021existence, awasthi2021extended} but also in the general \textit{distributionally robust optimization} literature, e.g.,  \cite{MR3959085}. Previous studies of robust classifiers use \textit{the universal $\sigma$-algebra} not only to formulate optimization problems rigorously, but also as a feasible search space for robust classifiers. The proofs of these existence results rely on the pointwise topology of a sequence of universally measurable sets, the weak topology on the space of probability measures, and  lower semi-continuity properties of $\overline{R}_{\varepsilon}(\cdot)$. The (universal) measurability of a minimizer is then guaranteed immediately by the definition of the universal $\sigma$-algebra. We want to emphasize that all the works \cite{pydi2021the, frank2022consistency, frank2022existence, awasthi2021existence, awasthi2021extended} prove their results in the binary ($K=2$) classification setting where $\X$ is a subset of Euclidean space.

In contrast to the closed-ball model formulation, the objective 
in \eqref{def:DRO_model} is well-defined for all Borel probability measures $\widetilde \mu$ and all $f\in \F$, as has been discussed in previous sections. The papers \cite{pydi2021the, frank2022consistency, frank2022existence, awasthi2021existence, awasthi2021extended} can only relate, in the binary case, problems \eqref{def:DRO_model} and \eqref{def:closed-ball_model-corrected} when problem \eqref{def:DRO_model} is appropriately extended to the universal $\sigma$-algebra, yet it is not clear that such extension is necessary. For concreteness, we summarize some of the results in those works in the following theorem.
\begin{theorem}[\cite{pydi2021the, awasthi2021existence, awasthi2021extended, frank2022existence}]
Suppose $K=2$ and $\overline{\mu} \in \overline{\mathcal{P}}(\mathcal{Z})$. Then, for any $f \in \mathcal{F}$, we have $\sup_{\widetilde{x} \in \overline{B}_{\varepsilon}(x)} \{ 1 - f_i(\widetilde{x}) \} \in \mathcal{U}(\mathcal{Z})$ and
\begin{equation*}
    \sum_{i=1}^2 \int_{\mathcal{X}} \sup_{\widetilde{x} \in \overline{B}_{\varepsilon}(x)} \{ 1 - f_i(\widetilde{x}) \} d\overline{\mu}_i(x) = \sup_{\widetilde{\mu} \in \overline{\mathcal{P}}(\mathcal{Z})} \left\{ R(f, \widetilde{\mu}) - C(\overline{\mu}, \widetilde{\mu}) \right\},
\end{equation*}
where $C$ is defined in terms of the cost $c_\varepsilon$ from \eqref{def:CostEpsilon}.

Assume further that $(\mathcal{X}, d) = (\mathbb{R}^n, || \cdot ||)$. Then, it holds that $\sup_{\widetilde{x} \in \overline{B}_{\varepsilon}(\cdot )} \{ 1 - f_i(\widetilde{x}) \} $ is universally measurable for any $f \in \mathcal{U}(\mathcal{Z})$ and each $i$. In addition, there exists a minimizer of the objective in \eqref{def:closed-ball_model-corrected} in the class of soft-classifiers that are universally measurable. Finally, \eqref{def:closed-ball_model-corrected} and \eqref{def:DRO_model} are equivalent, provided that the latter is interpreted as an optimization problem over the space of universally measurable soft-classifiers.  
\end{theorem}

In this paper, we use the universal $\sigma$-algebra to rigorously define the objective function in \eqref{def:closed-ball_model-corrected}, but we will only consider elements in $\F$ (thus, Borel measurable soft-classifiers) as feasible classifiers. Indeed, based on some of our previous results in \cite{trillos2023multimarginal}, we prove the existence of Borel measurable robust classifiers of \eqref{def:DRO_model} for general lower semi-continuous $c$ satisfying Assumption \ref{assump:CostStructure} only. Then, back to the closed-ball model, we prove the existence of Borel robust classifiers of \eqref{def:closed-ball_model-corrected}. When we specialize our results to the binary classification setting (i.e., $K=2$), we obtain the following improvement upon the results from \cite{Bhagoji2019LowerBO,Pydi2021AdversarialRV,frank2022existence}.



\begin{corollary}
\label{cor:BinaryCase}
Let $K=2$ and let $f^* \in \F$ be any solution to the problem \eqref{def:closed-ball_model-corrected}. Then, for Lebesgue a.e. $t \in [0,1]$, the pair $(\mathds{1}_{\{ f^*_1 \geq t \}} , \mathds{1}_{\{ f_1^* \geq t \}^c})$ is also a solution to \eqref{def:closed-ball_model-corrected}.

In particular, there exist solutions to the problem 
\[ \min_{A \in \mathfrak{B}(\X)}  \int_{\X} \sup_{ \tilde x \in \overline{B}_\varepsilon(x)}  \mathds{1}_{A^c}(\tilde x) d\overline{\mu}_1(x) + \int_{\X} \sup_{ \tilde x \in \overline{B}_\varepsilon(x)}  \mathds{1}_{A}(\tilde x) d\overline{\mu}_2(x).    \]
\end{corollary}
Notice that Corollary \ref{cor:BinaryCase} implies, for the binary case, the existence of robust hard-classifiers for the adversarial training problem, a property shared with the nominal risk minimization problem \eqref{eq:UnrobustRiskMinimization} that we discussed at the beginning of section \ref{preminaries}. Analogous results on the equivalence of the hard-classification and soft-classification problems in adversarial training under the binary setting have been obtained in \cite{Pydi2021AdversarialRV, pydi2021the, bungert2023geometry,  Trillos2020AdversarialCN}. Unfortunately, when the number of classes is such that $K>2$, the hard-classification and soft-classification problems in adversarial training may not be equivalent, as has been discussed in Section 5.2 of our work \cite{trillos2023multimarginal}.

In light of Theorem \ref{thm:AllmodelsSame}, one can conclude from Corollary \ref{cor:BinaryCase} that for all but countably many $\varepsilon>0$ the problem 
\[ \min_{A \in \mathfrak{B}(\X)}  \int_{\X} \sup_{ \tilde x \in {B}_\varepsilon(x)}  \mathds{1}_{A^c}(\tilde x) d{\mu}_1(x) + \int_{\X} \sup_{ \tilde x \in {B}_\varepsilon(x)}  \mathds{1}_{A}(\tilde x) d {\mu}_2(x)    \]
admits solutions; notice that the above is the open-ball version of the optimization problem in Corollary \ref{cor:BinaryCase}. However, notice that the results in \cite{bungert2023geometry} guarantee existence of solutions for \textit{all} values of $\varepsilon$. It is interesting to note that the technique used in \cite{bungert2023geometry} can not be easily adapted to the multiclass case $K>2$. Specifically, it does not seem to be straightforward to generalize Lemma C.1 in \cite{bungert2023geometry} to the multiclass case. For example, if one used the aforementioned lemma to modify the coordinate functions $f_i$ of a multiclass classifier $f$, one could end up producing functions for which their sum may be greater than one for some points in $\X$, thus violating one of the conditions for belonging to $\F$.



We observe, on the other hand, that the total variation regularization interpretation for the open ball model in the binary case discussed in \cite{bungert2023geometry} continues to hold in the multiclass case. To make this connection precise, let us introduce the non-local TV functionals:
\begin{equation*} 
\pretv_{\varepsilon}(f_i, \mu_i):= \frac{1}{\varepsilon} \sum_{i \in \Y} \int_{\mathcal{X}} \sup_{\widetilde{x} \in B_{\varepsilon}(x)} \{ f_i(x) -  f_i(\widetilde{x}) \} d\mu_i(x).
\end{equation*}
It is then straightforward to show that problem \eqref{def:open-ball_model} is equivalent to 
\begin{equation}
\inf_{f \in \F}  \sum_{i=1}^K \int_{\X}(1 - f_i(x)) d\mu_i(x)       + 
 \varepsilon \sum_{i=1}^K  \pretv_{\varepsilon}(f_i, \mu_i),
    \label{def: pre_TV}
\end{equation}
which can be interpreted as a total variation minimization problem with fidelity term. Indeed, the fidelity term in the above problems is the nominal (unrobust) risk $R(f, \mu)$. On the other hand,  the functional $ \pretv_{\varepsilon}(\cdot, \mu_i)$ is a non-local total variation functional in the sense that it is convex, positive $1$-homogeneous, invariant under addition of constants to the input function and is equal to zero when its input is a constant function. Moreover, in the case $(\X,d) = (\R^d, \lVert \cdot \rVert)$ and when $d\mu_i(x)=\rho_i(x) dx$ for a smooth function $\rho_i$, one can see that, for small $\varepsilon>0$,
\begin{equation*}
   \pretv_{\varepsilon}(f_i, \mu_i) \approx \int_{\mathcal{X}} \left| \nabla f_i(x) \right| \rho_i(x)dx,
\end{equation*}
when $f_i$ is a smooth enough function. The functional $\pretv_{\varepsilon}(f_i, \mu_i)$ is thus connected to more standard notions of (weighted) total variation in Euclidean space. This heuristic can be formalized further via variational tools, as has been done recently in \cite{bungert2022gamma}.

Total variation regularization with general TV functionals is an important methodology in imaging, and also in unsupervised and supervised learning on graphs, where it has been used for community detection, clustering, and graph trend-filtering; e.g., see \cite{BertozziHuLaurent, MBO_Scheme_on_Graphs, MR3207626, doi:10.1137/16M1070426, luo2017convergence, caroccia2020mumford, MR4163472, MR4106615, MR4419604,GTMurray2017} and references therein.

\section{Distributional-perturbing model and its generalized barycenter formulation}\label{section : existence dual potential}

In this section we introduce some tools and develop a collection of technical results that we use in section \ref{section : Borel robust classifier} when proving Theorem \ref{thm : measurable robust classifier}.

\subsection{Generalized barycenter and MOT problems}
In our work \cite{trillos2023multimarginal} we introduced the following \textit{generalized barycenter problem}. Given $\mu\in \mathcal{P}(\Z)$, we consider the optimization problem
\begin{equation}
\label{eq:generalized_barycenter}
    \inf_{\lambda, \widetilde{\mu}_1, \ldots, \widetilde{\mu}_K} \left\{ \lambda(\mathcal{X})+ \sum_{i \in \Y} C(\mu_i, \widetilde{\mu}_i) : \text{ $\lambda \geq\widetilde{\mu}_i$ for all $i \in \Y$}\right\}.
\end{equation}
In the above, the infimum is taken over positive (Borel) measures $\tilde{\mu}_1, \dots, \tilde{\mu}_K$ and $\lambda$ satisfying the constraints $\lambda \geq \tilde{\mu}_i$ for all $i \in \Y$. This constraint must be interpreted as: $\lambda(A) \geq \tilde{\mu}_i(A)$ for all $A \in \mathfrak{B}(\X)$. Problem \eqref{eq:generalized_barycenter} can be understood as a generalization of the standard (Wasserstein) barycenter problem studied in \cite{AguehCarlier}.  Indeed, if all measures $\mu_1, \dots, \mu_K$ had the same total mass and the term $\lambda(\X)$ in \eqref{eq:generalized_barycenter} was rescaled by a constant $\alpha \in (0,\infty)$, then, as $\alpha \rightarrow \infty$, the resulting problem would recover the classical barycenter problem with pairwise cost function $c$. As stated, one can regard \eqref{eq:generalized_barycenter} as a partial optimal transport barycenter problem: we transport each $\mu_i$ to a part of $\lambda$ while requiring the transported masses to overlap as much as possible (this is enforced by asking for the term $\lambda(\X)$ to be small).

We recall a result from \cite{trillos2023multimarginal} which essentially states that the generalized barycenter problem \eqref{eq:generalized_barycenter} is dual to \eqref{def:DRO_model}. 

\begin{theorem}[Proposition 7 and Corollary 32 in \cite{trillos2023multimarginal}]
\label{thm:adversary_part}
Suppose that $c$ satisfies Assumption \ref{assump}. Then
\begin{equation*}\label{eq : adversarial_generalized_connection}
    \eqref{def:DRO_model} = 1 - \eqref{eq:generalized_barycenter}.
\end{equation*}
Furthermore, the infimum of \eqref{eq:generalized_barycenter} is attained. In other words, there exists $(\lambda^*, \widetilde{\mu}^*)$ which minimizes \eqref{eq:generalized_barycenter}.
\end{theorem}

Like classical barycenter problems, \eqref{eq:generalized_barycenter} has an equivalent multimarginal optimal transport (MOT) formulation. To be precise, we use a \textit{stratified} multimarginal optimal transport problem to obtain an equivalent reformulation of \eqref{eq:generalized_barycenter}.

\begin{theorem}[Proposition 14 and 15 in \cite{trillos2023multimarginal}]\label{thm:generalized_barycenter_MOT} 
Suppose that $c$ satisfies Assumption \ref{assump}.
Let $S_K :=\{A\subseteq \Y : A \neq \emptyset\}$. Given $A \in S_K$, define $c_A: \mathcal{X}^K \to [0,\infty]$ as $c_A(x_1, \ldots, x_K):=\inf_{x'\in \mathcal{X}} \sum_{i\in A} c(x', x_i)$. 

Consider the problem:
\begin{equation}\label{eq:multimarginal_decomposed}
\begin{aligned}
    &\inf_{ \{ \pi_A : A \in S_K\} } \sum_{A\in S_K} \int_{\mathcal{X}^K} \big(c_A(x_1,\ldots, x_K)+1\big) d\pi_A(x_1, \dots, x_K)\\
    &\textup{s.t.} \sum_{A\in S_K(i)}\mathcal{P}_{i\,\#}\pi_A=\mu_i  \textup{ for all } i\in \Y,
\end{aligned}
\end{equation}
where $\mathcal{P}_i$ is the projection map $\mathcal{P}_i: (x_1, \ldots, x_K)\mapsto x_i$, and $S_K(i) := \{ A\in S_K \: : \: i \in A \}$. Then \eqref{eq:generalized_barycenter} = \eqref{eq:multimarginal_decomposed}. Also, the infimum in \eqref{eq:multimarginal_decomposed} is attained.
\end{theorem}

\begin{remark}
Even though $c_A$ and $\pi_A$ above are defined over $\mathcal{X}^K$, only the coordinates $i$ where $i \in A$ actually play a role in the optimization problem. Also, notice that \eqref{eq:multimarginal_decomposed} is not a standard MOT problem since in \eqref{eq:multimarginal_decomposed} we optimize over several couplings $\pi_A$ (each with its own cost function $c_A$) that are connected to each other via the marginal constraints. We refer to this type of problem as a stratified MOT problem.  
\end{remark}

In the following theorem we discuss the duals of the generalized barycenter problem and its MOT formulation. The notions of $c$-transform and $\overline{c}$-transform, whose definition we revisit in Appendix \ref{appendix : c-transform}, plays an important role in these results.

\begin{theorem}[Proposition 22 and Proposition 24 in \cite{trillos2023multimarginal}]
\label{thm:learner_part}
Suppose that $c$ satisfies Assumption \ref{assump}. Let $\mathcal{C}_b(\mathcal{X})$ be the set of bounded real-valued continuous functions over $\X$. The dual of \eqref{eq:generalized_barycenter} is
\begin{equation}\label{eq:barycenter_dual}
\begin{aligned}
    &\sup_{f_1, \ldots, f_K \in \mathcal{C}_b(\mathcal{X})}  \sum_{i\in \Y} \int_{\mathcal{X}} f_i^c(x_i) d\mu_i(x_i)\\
    &\textup{s.t.}\; f_i(x)\geq 0, \;  \sum_{i\in \Y} f_i(x)\leq 1,\; \textup{for all }\, x \in \mathcal{X},\,\, i \in \{1,\ldots, K\},
\end{aligned}
\end{equation}
and there is no duality gap between primal and dual problems. In other words, \eqref{eq:generalized_barycenter} = \eqref{eq:barycenter_dual}. In the above, $f_i^c$ denotes the $c$-transform of $f_i$ as introduced in Definition \ref{def:cTransforms}.

The dual of \eqref{eq:multimarginal_decomposed} is
\begin{equation}\label{eq:mot_decomposed_dual}
\begin{aligned}
    &\sup_{g_1, \ldots, g_K\in \mathcal{C}_b(\mathcal{X})} \sum_{i\in \Y} \int_{\mathcal{X}} g_i(x_i) d\mu_i(x_i)\\
    &\textup{s.t.} \;\sum_{i \in A } g_i(x_i)\leq 1+c_A(x_1, \dots, x_K) \textup{ for all } (x_1, \dots, x_K) \in \mathcal{X}^K, \,\, A\in S_K,
\end{aligned}
\end{equation}
and there is no duality gap between primal and dual problems. In other words, \eqref{eq:multimarginal_decomposed} = \eqref{eq:mot_decomposed_dual}.

If in addition the cost function $c$ is bounded and Lipschitz, then \eqref{eq:mot_decomposed_dual} is achieved by $g \in \mathcal{C}_b(\mathcal{X})^K$. Also, 
for $f$ feasible for \eqref{eq:barycenter_dual}, $g':=f^c$ is feasible for \eqref{eq:mot_decomposed_dual}. Similarly, for $g$ feasible for \eqref{eq:mot_decomposed_dual}, $f'=\max \{ g, 0 \}^{\overline{c}}$ is feasible for \eqref{eq:barycenter_dual}. Therefore, the optimization of \eqref{eq:mot_decomposed_dual} can be restricted to non-negative $g$ satisfying $g_i=g_i^{\overline{c}c}$, or $0 \leq g_i \leq 1$ for all $i \in \Y$. The notions of $c$-transform and $\overline{c}$-transform are introduced in Definition \ref{def:cTransforms}.
\end{theorem}

\begin{remark}
\label{rem:DualityGandDRO}
Combining Theorem \ref{thm:adversary_part}, Theorem \ref{thm:generalized_barycenter_MOT}, and Theorem \ref{thm:learner_part} we conclude that $1-\eqref{eq:mot_decomposed_dual} = \eqref{def:DRO_model}$.


\end{remark}

\begin{remark}
\label{rem:DualInLinfty}
A standard argument in optimal transport theory shows that
problem \eqref{eq:mot_decomposed_dual} is equivalent to 
\begin{equation}\label{eq:mot_decomposed_dual_Linfty}
\begin{aligned}
    &\sup_{g_1, \ldots, g_K} \sum_{i\in \Y} \int_{\mathcal{X}} g_i(x_i) d\mu_i(x_i),
\end{aligned}
\end{equation}
where the sup is taken over all $(g_1, \dots, g_K) \in \prod_{i \in \Y} L^\infty(\X; \mu_i)$ satisfying: for any $A \in S_K$, 
\begin{equation*}
    \sum_{i \in A} g_i(x_i) \leq 1 + c_A(x_1, \dots, x_K)
\end{equation*}
for $\otimes_{i } \mu_i$-almost every tuple $(x_1, \dots, x_K)$. Indeed, notice that since \eqref{eq:mot_decomposed_dual} has already been shown to be equal to \eqref{eq:multimarginal_decomposed}, the claim follows from the observation that any feasible $g_1, \dots, g_K$ for \eqref{eq:mot_decomposed_dual_Linfty} satisfies the condition
\[
\sum_{i \in \Y} \int_{\X} g_i(x_i) d\mu_i(x_i) \leq \sum_{A \in S_K}\int_{\X^K}(1+ c_A(x_1, \dots, x_K )) d\pi_A(x_1, \dots, x_K)
\]
for every $\{ \pi_A\}_{A \in S_K}$ satisfying the constraints in \eqref{eq:multimarginal_decomposed}.

\end{remark}

\subsection{Existence of optimal dual potentials $g$ for general lower semi-continuous costs}\label{sec:g_existence}
We already know from the last part in Theorem \ref{thm:learner_part} that if $c$ is bounded and Lipschitz, then there is a feasible $g \in \mathcal{C}_b(\mathcal{X})^K$ that is optimal for \eqref{eq:mot_decomposed_dual}. In this subsection we prove an analogous existence result in the case of a general lower semi-continuous cost function $c$ satisfying Assumption \ref{assump}. More precisely, we prove existence of maximizers for \eqref{eq:mot_decomposed_dual_Linfty}. We start with an approximation result.

\begin{lemma}
\label{lem:ApproximationCost}
Let $c$ be a cost function satisfying Assumption \ref{assump}. For each $n \in \N$ let
\[ c_n(x,x'):= \min \{ \tilde{c}_n(x, x') , n \}, \]
where
\[ \tilde{c}_n(x, x') := \inf_{(\tilde x, \tilde x ') \in \X \times \X} \{   c(\tilde x , \tilde x ')  + nd(x, \tilde x ) + n d(x', \tilde x ' ) \}. \]
Then the following properties hold:
\begin{enumerate}
\item $c_n$ is bounded and Lipschitz.
\item $c_n \leq c_{n+1} \leq c$ and $\tilde c_n \leq \tilde c_{n+1}$ for all $n \in \N$.
\item $\lim_{n \rightarrow \infty} c_n(x,x') = c(x,x')$ for all $(x,x') \in \X \times \X$.
\end{enumerate}

\end{lemma}
\begin{proof}
Items 1. and 2. are straightforward to prove. To prove item 3., notice that due to the monotonicity of the cost functions we know that $\lim_{n \rightarrow \infty} c_n(x,x')$ exists in $[0,\infty]$ and $\lim_{n \rightarrow \infty} c_n(x,x')  \leq c(x,x') $. If $\lim_{n \rightarrow \infty} c_n(x,x')=\infty$, then we would be done. Hence, we may assume that $\lim_{n \rightarrow \infty} c_n(x,x') < \infty$. From the definition of $c_n$ it then holds that $\lim_{n \rightarrow \infty} \tilde c_n(x,x') = \lim_{n \rightarrow \infty}  c_n(x,x')  < \infty$. Let $(x_n, x_n' ) \in \X \times \X$ be such that
\[   c(x_n, x_n') + nd(x,x_n) + nd(x',x_n') \leq \tilde{c}_n(x,x') + \frac{1}{n}. \]
Since $c(x_n, x_n') \geq 0$, the above implies that $\lim_{n \rightarrow \infty} d(x,x_n) =0 $ and $\lim_{n \rightarrow \infty} d(x',x_n') =0$. Indeed, if this was not the case, then we would contradict $\lim_{n \rightarrow \infty} \tilde c _n(x,x') <\infty$. By the lower semicontinuity of the cost function $c$ we then conclude that
\begin{align*}
    c(x, x') \leq \liminf_{n \rightarrow \infty } c(x_n, x_n')  & \leq \liminf_{n \rightarrow \infty} c(x_n, x_n') + nd(x,x_n) + nd(x',x_n') \leq \liminf_{n \rightarrow \infty} \tilde c_n(x,x') 
    \\ &= \lim_{n \rightarrow \infty } c_n(x,x') \leq c(x,x'), 
\end{align*}
from where the desired claim follows.
\end{proof}

\begin{lemma}\label{lem : monotone convergence c_A}
Let $c$ be a cost function satisfying Assumption \ref{assump}, and let $c_n$ be the cost function defined in Lemma \ref{lem:ApproximationCost}. For each $A \in S_K$, let 
\begin{equation*}
    c_{A,n}(x_A) := \inf_{x' \in \X} \sum_{i \in A} c_n(x', x_i), \text{ and }  c_A(x_A) := \inf_{x' \in \X} \sum_{i \in A} c(x', x_i),
\end{equation*}
where we use the shorthand notation $x_A= (x_i)_{i \in A}$. Then $c_{A,n}$ monotonically converges toward $c_A$ pointwise for all $A \in S_K$, as $n \rightarrow \infty$. 
\end{lemma}

\begin{proof}
Fix $A \in S_K$ and $x_A:=(x_i)_{i \in A} \in \mathcal{X}^{|A|}$. From Lemma \ref{lem:ApproximationCost} it follows $c_{A,n} \leq c_{A, n+1} \leq c_{A} $. Therefore, for a given $x_A$, $\lim_{n \rightarrow \infty} c_{A,n}(x_A)  $ exists in $[0,\infty]$ and is less than or equal to $c_A(x_A)$. If the limit is $\infty$, we are done. We can then assume without the loss of generality that $\lim_{n \rightarrow \infty} c_{A, n}(x_A) <\infty$. We can then find sequences $\{ x_{n,i} \}_{n \in \N}$, $\{  x_{n,i}' \}_{n \in \N}$, and $\{ x_n' \}_{n \in \N}$ such that for all large enough $n \in \N$ 
\[ \sum_{i \in A} c( x_{n,i}',x_{n,i}) + n ( \sum_{i\in A} ( d(x_{n,i}, x_i) + d(x_{n,i}', x_n') ))  \leq c_{A,n}(x_A) +  \frac{1}{n}.   \]
From the above we derive that $\lim_{n \rightarrow \infty} d(x_{n,i}', x_n') =0$ and $\lim_{n \rightarrow \infty} d(x_{n,i}, x_i) =0$. Hence, it follows that $\limsup_{n \rightarrow \infty} c( x_{n,i}',x_{n,i}) < \infty$. Combining the previous facts with Assumption \ref{assump}, we conclude that $\{ x_{n}' \}_{n \in \N}$ is precompact, and thus, up to subsequence (that we do not relabel), we have $\lim_{n \rightarrow \infty }d(x'_{n}, \hat{x}) =0 $ for some $\hat x \in \X$. Combining with  $\lim_{n \rightarrow \infty} d(x_{n,i}', x_n') =0$ we conclude that $\lim_{n \rightarrow \infty} d(x_{n,i}', \hat x) =0$ for all $i \in A$. Using the lower semi-continuity of $c$ we conclude that
\[ c_A(x_A) \leq \sum_{i\in A} c(\hat{x}, x_i) \leq \liminf_{n \rightarrow \infty} \sum_{i\in A} c(x_{n,i}', x_{n,i}) \leq \lim_{n \rightarrow \infty} c_{A,n}(x_A) \leq c_{A}(x_A).      \]

\end{proof}


\begin{proposition}\label{prop:optimality_weak_convergence_dual}
Let $c$ be a cost function satisfying Assumption \ref{assump}. Then there exists a solution for \eqref{eq:mot_decomposed_dual_Linfty}.
\end{proposition}

\begin{proof}
Let $\{c_n\}_{n \in \N}$ be the sequence of cost functions introduced in Lemma \ref{lem:ApproximationCost}. Notice that for each $n \in \mathbb{N}$ there is a solution $g^n=(g^n_1, \dots, g^n_K) \in \mathcal{C}_b(\mathcal{X})^K$ for the problem \eqref{eq:mot_decomposed_dual} (with cost $c_n$) that can be assumed to satisfy $0 \leq g^n_i \leq 1$ for each $i \in \Y$. Therefore, for each $i \in \Y$ the sequence $\{ g_n^i \}_{n \in \N}$ is weakly$^*$ precompact in $L^\infty(\X; \mu_i)$. This implies that there exists a subsequence of $\{g^n\}_{n \in \N}$ (not relabeled) for  which $g^n$ weakly$^*$ converges toward some $g^* \in \prod_{i \in \Y} L^{\infty}(\mathcal{X}; \mathbb{R}, \mu_i)$, which would necessarily satisfy $0 \leq g_i^*\leq 1 $ for all $i \in \Y$; see section \ref{sec:WeakStarTopology} for the definition of weak$^*$ topologies. We claim that this $g^*$ is feasible for \eqref{eq:mot_decomposed_dual_Linfty}. Indeed, by Lemma \ref{lem : monotone convergence c_A} we know that $c_{A,n} \leq c_A$ for all $A \in S_K$. In particular, since $c_{A, n} \leq c_A$, and $\sum_{i \in A} g^n_i(x_i) \leq 1 + c_{A,n}\leq 1 + c_A$ for all $A \in S_K$ and all $n \in \mathbb{N}$, it follows that $\sum_{i \in A} g^*_i(x_i) \leq 1 + c_A$,  $\otimes_{i } \mu_i$-almost everywhere, due to the weak$^*$ convergence of $g_i^n$ toward $g_i^*$. This verifies that $g^*$ is indeed feasible for \eqref{eq:mot_decomposed_dual_Linfty}.

Let $\alpha_n$ and $\beta_n$ be the optimal values of \eqref{eq:generalized_barycenter} and \eqref{eq:mot_decomposed_dual}, respectively, for the cost $c_n$. Likewise, let $\alpha$ and $\beta$ be the optimal values of \eqref{eq:generalized_barycenter} and \eqref{eq:mot_decomposed_dual}, respectively, for the cost $c$. Recall that, thanks to Theorem \ref{thm:generalized_barycenter_MOT} and Theorem \ref{thm:learner_part}, we have $\alpha_n = \beta_n $ for all $n \in \N$ and $\alpha=\beta$. Suppose for a moment that we have already proved that $\lim_{n \rightarrow \infty} \alpha_n = \alpha$. Then we would have
\[  
\sum_{i \in \Y} \int_{\X} g_i^*(x) d\mu_i(x)  = \lim_{n \rightarrow \infty} \sum_{i \in \Y} \int_{\X} g_i^n(x) d\mu_i(x) = \lim_{n \rightarrow \infty} \beta_n   = \lim_{n \rightarrow \infty} \alpha_n = \alpha,    
\]
which would imply that $g^*$ is optimal for \eqref{eq:mot_decomposed_dual_Linfty}.

It thus remains to show that $\lim_{n \rightarrow \infty} \alpha_n = \alpha$. Given that $c_n \leq  c_{n+1} \leq c$, it follows that $\alpha_n \leq \alpha_{n+1} \leq \alpha$. In particular, the limit $\lim_{n \rightarrow \infty} \alpha_n $ exists in $[0,\infty]$ and must satisfy $\lim_{n \rightarrow \infty} \alpha_n \leq \alpha$. If the limit is $\infty$, then there is nothing to prove. Thus we can assume without the loss of generality that $\alpha_\infty:= \lim_{n \rightarrow \infty} \alpha_n < \infty$. 

Let $\lambda^n$ and $\widetilde{\mu}_1^n,\ldots, \widetilde{\mu}_K^n$ be an optimal solution of \eqref{eq:generalized_barycenter} with the cost $c_n$ and let $\pi^n_i$ be a coupling realizing $C(\mu_i, \tilde{\mu}_i^n)$. We first claim that $\{\tilde{\mu}_i^n \}_{n \in \N}$ is weakly precompact for each $i\in \Y$. To see this, notice that for every $n$ we have $\tilde{\mu}_i^n(\X)= \mu_i(\X)\leq 1$, for otherwise $C(\mu_i, \widetilde \mu_i^n )=\infty$. Thus, by Prokhorov theorem it is enough to show that for every $\eta>0$ there exists a compact set $\mathcal{K} \subseteq \X$ such that $\tilde{\mu}^n_i( \X \setminus \mathcal{K} ) \leq C\eta$ for all $n \in \N$ and some $C$ independent of $n, \eta$ or $\mathcal{K}$. To see that this is true, let us start by considering a compact set $G$ such that $ \mu_i(G^c) \leq \eta$.
Let $n_0 \in \N $ be such that $n_0 -1 > \frac{1}{\eta}$. For $n \geq n_0$ we have
\[  \alpha_\infty\geq \alpha_n = \lambda_n(\X) + \sum_{i \in \Y}\int_{\X}\int_{\X} c_n(x_i, \tilde x_i) d \pi_i^n(x_i, \tilde x _i) \geq \int_{G}\int_{\X} c_{n_0}(x_i, \tilde x_i) d \pi_i^n(x_i, \tilde x _i) . \]
Consider the set
\[  \tilde{\mathcal{K}}:= \{ x \in \X \text{ s.t. }  \inf_{\tilde x \in G } c_{n_0}(x, \tilde x) \leq n_0 -1   \};\]
using the definition of $c_{n_0}$ and Assumption \ref{assump} it is straightforward to show that $ \tilde{\mathcal{K}} $ is a compact subset of $\X$. We see that $\alpha_\infty \geq \frac{1}{\eta} ( \tilde{\mu}_i^n (\tilde{\mathcal{K}}^c) - \mu_i(G^c) )$, from where we can conclude that $\tilde{\mu}_i^n(\tilde{\mathcal{K}}^c) \leq (\alpha_\infty + 1)\eta$ for all $n \geq n_0$.  We now consider a compact set $\hat {\mathcal{K}}$ for which $\tilde{\mu}_i^n(\hat{\mathcal{K}}^c) \leq \eta$ for all $n=1, \dots, n_0$, and set $\mathcal{K}:= \tilde {\mathcal{K}} \cup \hat{\mathcal{K}}$, which is compact. Then for all $n \in \N$ we have $\widetilde \mu_i^n(\mathcal{K}^c) \leq (\alpha_\infty +1) \eta$. This proves the desired claim. 

Now, without the loss of generality, we can assume that $\lambda^n$ has the form 
\[  
d \lambda^n(x) = \max_{i=1, \dots, K} \left\{ \frac{d\widetilde \mu_i^n}{d\overline{\mu}^n } (x) \right\} d \overline{\mu}^n(x),    
\]
where $\overline{\mu}^n(x) = \sum_{i=1}^K \widetilde \mu_i^n$. Indeed, notice that the above is the smallest positive measure greater than $\tilde{\mu}_1^n, \dots, \tilde{\mu}_K^n$. Given the form of $\lambda^n$ and the weak precompactness of each of the sequences $\{ \tilde{\mu}_i^n \}_{n \in \N}$, we can conclude that $\{ \lambda^n\}_{n \in \N}$ is weakly precompact and so are the sequences $\{\pi_i^n\}_{n \in \N}$. We can thus assume that, up to subsequence, $\widetilde \mu_i^n $ converges weakly toward some $\widetilde \mu_i$; $\pi_i^n$ converges weakly toward some $\pi_i \in \Gamma(\mu_i, \widetilde \mu_i)$; and $\lambda^n$ converges weakly toward some $\lambda$ satisfying $\lambda \geq \widetilde \mu_i$ for each $i\in \Y$. In particular, $\lambda, \tilde{\mu}_1, \dots, \tilde{\mu}_K$ is feasible for \eqref{eq:generalized_barycenter}.  

Therefore, for all $n_0 \in \N$ we have
\begin{align*}
\alpha \geq \alpha_\infty & = \lim_{n \rightarrow \infty }  \left(\lambda^n(\X) + \sum_{i\in \Y} \int_\X \int_\X c_n(x_i, \tilde x_i) d\pi_i^n(x_i, \tilde x_i) \right)\\ 
&\geq \lim_{n \rightarrow \infty} \left(\lambda^n(\X) + \sum_{i\in \Y} \int_\X \int_\X c_{n_0}(x_i, \tilde x_i) d\pi_i^n(x_i, \tilde x_i) \right)\\
& \geq  \lambda(\X) + \sum_{i\in \Y} \int_\X \int_\X c_{n_0}(x_i, \tilde x_i) d\pi_i(x_i, \tilde x_i).
\end{align*}
Sending $n_0\rightarrow \infty$, we can then use the monotone convergence theorem to conclude that
\[ \alpha \geq \alpha_\infty \geq \lambda(\X) + \sum_{i\in \Y} \int_\X \int_\X c(x_i, \tilde x_i) d\pi_i(x_i, \tilde x_i) \geq \lambda(\X) + \sum_{i\in \Y} C(\mu_i , \widetilde \mu_i)   \geq  \alpha. \]
This proves that $\alpha_\infty=\alpha$.
\end{proof}

\subsection{From dual potentials to robust classifiers for continuous cost functions}

Having discussed the existence of solutions $g^*$ for \eqref{eq:mot_decomposed_dual_Linfty}, we move on to discussing the connection between $g^*$ and solutions $f^*$ of problem \eqref{def:DRO_model}.



{
\begin{proposition}[Originally Corollary 33 in \cite{trillos2023multimarginal}; see correction in Corollary 4.7. and Remark 4.9 in \cite{NGTJakwangMattArxiv}]\label{prop:OptimalPair}
Let $c: \X \times \X \rightarrow [0,\infty]$ be a lower semi-continuous function and suppose that $(\widetilde{\mu}^*, g^*)$ is a solution pair for the generalized barycenter problem \eqref{eq:generalized_barycenter} and the dual of its MOT formulation \eqref{eq:mot_decomposed_dual_Linfty}. Let $f^*$ be defined as 
\begin{equation}
\label{eq:f_i*}
    f^*_i(\widetilde{x}):= \max \left\{ \sup_{x \in \spt(\mu_i)} \left\{ g^*_i(x) - c(x, \widetilde{x}) \right\}, 0 \right\},
\end{equation}
for each $i \in \mathcal{Y}$. 

If $f^*$ is Borel-measurable, then $(f^*, \widetilde{\mu}^*)$ is a saddle solution for the problem \eqref{def:DRO_model}. In particular, $f^*$ is a minimizer of \eqref{def:DRO_model}. 
\end{proposition}
}


The reason why we can not directly use Proposition \ref{prop:OptimalPair} to prove existence of solutions to \eqref{def:DRO_model} for arbitrary $c$ and $\mu$ is because it is a priori not guaranteed that $f_i^*$, as defined in \eqref{eq:f_i*}, is Borel measurable; notice that the statement in Proposition \ref{prop:OptimalPair} is conditional. If $\spt(\mu_i)$ was finite for all $i$, then the Borel measurability of $f_i^*$ would follow immediately from the fact that the maximum of finitely many lower semi-continuous functions is Borel; this is of course the case when working with empirical measures. Likewise, the Borel measurability of $f_i^*$ is guaranteed when $\mu$ is arbitrary and $c$ is a bounded Lipschitz function (in fact, it is sufficient for the cost to be continuous), as is discussed in Definitions 5.2 and 5.7 and Theorem 5.10 in \cite{Oldandnew}. However, nothing can be said about the Borel measurability of $f_i^*$ without further information on $g_i^*$ (which in general is unavailable) when $c$ is only assumed to be lower semi-continuous (as is the case for the cost $c_\varepsilon$ from \eqref{def:CostEpsilon})
 and $\spt(\mu_i)$ is an uncountable set.

Our strategy to prove Theorem \ref{thm : measurable robust classifier} in section \ref{sec:WellPosednes} will be to approximate an arbitrary cost function $c$ from below with a suitable sequence of bounded and Lipschitz cost functions $c_n$ (the costs defined in Lemma \ref{lem:ApproximationCost}), and, in turn, consider a limit of the robust classifiers $f_n^*$ associated to each of the $c_n$. This limit ($\limsup$, to be precise) will be our candidate solution for \eqref{def:DRO_model}.


\section{Proofs of our main results}\label{section : Borel robust classifier}
In this section we prove the existence of a Borel measurable robust classifier for problem \eqref{def:DRO_model} when $c$ is an arbitrary lower semi-continuous cost function satisfying Assumption \ref{assump}. We also establish the existence of minimizers of \eqref{def:closed-ball_model-corrected} and establish Theorem \ref{thm:AllmodelsSame} and Corollary \ref{cor:BinaryCase}.

\subsection{Well-posedness of the DRO model}
\label{sec:WellPosednes}

\begin{proof}[Proof of Theorem \ref{thm : measurable robust classifier}]

Let $\{c_n\}_{n \in \N}$ be the sequence of cost functions converging to $c$ from below defined in Lemma \ref{lem:ApproximationCost}. For each $n \in \mathbb{N}$, we use Theorem \ref{thm:learner_part} and let $g^n=(g^n_1, \dots, g^n_K) \in \mathcal{C}_b(\mathcal{X})^K$ be a solution of \eqref{eq:mot_decomposed_dual} with cost $c_n$; recall that we can assume that $0 \leq g^n_i \leq 1$.  In turn, we use $g^n$ and the cost $c_n$ to define $f^n:=(f^n_1, \dots, f^n_K)$ following \eqref{eq:f_i*}. Since the $g_i^n$ and $c_n$ are continuous, and given that the pointwise supremum of a family of continuous functions is lower semi-continuous, we can conclude that $f^n_i$ is lower semi-continuous and thus also Borel measurable for each $n \in \mathbb{N}$. Thanks to Proposition \ref{prop:OptimalPair}, $f^n$ is optimal for \eqref{eq:barycenter_dual} with cost function $c_n$.

From the proof of Proposition \ref{prop:optimality_weak_convergence_dual}, we know that there exists a subsequence (that we do not relabel) such that the $g_i^n$ converge in the weak$^*$ topology, as $n \rightarrow \infty$, toward limits $g_i^*$ that form a solution for \eqref{eq:mot_decomposed_dual_Linfty} with cost $c$.
Using this subsequence, we define $f^* \in \mathcal{F}$ according to
\begin{equation}
   f^*_i(\widetilde{x}) := \limsup_{n \to \infty} f^n_i(\widetilde{x}), \quad \tilde x \in \X. 
   \label{eq:Limsupf_i}
\end{equation}
Notice that each $f^*_i$ is indeed Borel measurable since it is the $\limsup$ of Borel measurable functions. In addition, notice that $0 \leq f^*_i \leq 1$, due to the fact that $0\leq f^n_i \leq 1$ for all $i \in \Y$ and all $n \in \N$. We'll conclude by proving that $f^*$ is a solution for \eqref{def:DRO_model}.


Let $\widetilde{\mu}\in \mathcal{P}(\mathcal{Z})$ be an arbitrary Borel probability measure with $C(\mu, \widetilde \mu)<\infty$. For each $i \in \Y$ let $\pi_i$ be an optimal coupling realizing the cost $C(\mu_i, \widetilde{\mu}_i)$. Then
\begin{align*}
    &R(f^*, \widetilde{\mu}) - C(\mu, \widetilde{\mu})\\ 
    &= 1 - \sum_{i \in \Y} \int_{\mathcal{X}} f^*_i(\widetilde{x}) d\widetilde{\mu}_i(\widetilde{x}) - \sum_{i \in \Y} \int_{\mathcal{X} \times \mathcal{X}} c(x, \widetilde{x}) d\pi_i(x, \widetilde{x})\\
    &= 1 - \sum_{i \in \Y} \int_{\mathcal{X} \times \mathcal{X}} \left( f^*_i(\widetilde{x}) +  c(x, \widetilde{x}) \right) d\pi_i(x, \widetilde{x})\\
    &= 1 - \sum_{i \in \Y} \int_{\mathcal{X} \times \mathcal{X}} \left( \limsup_{n \to \infty} \max\left\{ \sup_{x' \in \spt(\mu_i)} \left\{ g^n_i(x') - c_n(x', \widetilde{x}) \right\}, 0 \right\} +  c(x, \widetilde{x}) \right) d\pi_i(x, \widetilde{x})\\
    &\leq 1 - \sum_{i \in \Y} \int_{\mathcal{X} \times \mathcal{X}} \left( \limsup_{n \to \infty} \sup_{x' \in \spt(\mu_i)} \left\{ g^n_i(x') - c_n(x', \widetilde{x}) \right\}+  c(x, \widetilde{x}) \right) d\pi_i(x, \widetilde{x})
\end{align*}
where the last inequality follows from the simple fact that $-\max\{a, 0\} \leq - a$ for any $a \in \mathbb{R}$. Choosing $x' = x$ in the $\sup$ term (notice that indeed $x$ can be assumed to belong to $\spt(\mu_i)$ since $\pi_i$ has first marginal equal to $\mu_i$), and applying reverse Fatou's lemma, we find that
\begin{align*}
    R(f^*, \widetilde{\mu}) - C(\mu, \widetilde{\mu}) & \leq 1 - \sum_{i \in \Y} \int_{\mathcal{X} \times \mathcal{X}} \limsup_{n \to \infty} \left\{ g^n_i(x) - c_n(x, \widetilde{x})  +  c_n(x, \widetilde{x}) \right\} d\pi_i(x, \widetilde{x})\\
    &= 1 - \sum_{i \in \Y} \int_{\mathcal{X} \times \mathcal{X}} \limsup_{n \to \infty} \left\{ g^n_i(x) \right\} d\pi_i(x, \widetilde{x})
    \\& = 1 - \sum_{i \in \Y} \int_{\mathcal{X} } \limsup_{n \to \infty} \left\{ g^n_i(x) \right\} d\mu_i(x)\\
    &\leq 1 - \limsup_{n \to \infty} \sum_{i \in \Y} \int_{\mathcal{X}}  g^n_i(x) d\mu_i(x)\\ 
    &= 1 - \sum_{i \in \Y} \int_{\mathcal{X}} g^*_i(x) d\mu_i(x)
    \\&= R^*_{DRO}, 
\end{align*} 
where the third equality follows from the weak$^*$ convergence of $g^n_i$ toward $g^*_i$ and the last equality follows from remark \ref{rem:DualityGandDRO} and the fact that $g^*$ is a solution for \eqref{eq:mot_decomposed_dual_Linfty} (combined with remark \ref{rem:DualInLinfty}). Taking the sup over $\widetilde \mu \in \mathcal{P}(\Z)$, we conclude that
\[ \sup_{ \widetilde \mu \in \mathcal{P}(\Z) } \{ R(f^*, \widetilde \mu ) - C(\mu, \widetilde \mu) \}  \leq R^*_{DRO} ,   \]
and thus $f^*$ is indeed a minimizer of \eqref{def:DRO_model}.


Let now $\widetilde\mu^*$ be a solution of \eqref{eq:generalized_barycenter} (which exists due to Theorem \ref{thm:adversary_part}). The fact that $(\widetilde \mu^*, f^*)$ is a saddle for \eqref{def:DRO_model} follows from the above computations and the fact that by Theorem \ref{thm:adversary_part} and Corollary 32 in \cite{trillos2023multimarginal} we have
\begin{align*}
    R^*_{DRO} =  \sup_{\widetilde{\mu} \in \mathcal{P}(\mathcal{Z})} \inf_{f \in \mathcal{F}}  \left\{ R(f, \widetilde{\mu}) - C(\mu, \widetilde{\mu}) \right\}= \inf_{f \in \mathcal{F}}  \left\{ R(f, \widetilde{\mu}^*) - C(\mu, \widetilde{\mu}^*) \right\}.
\end{align*}

\end{proof}

The next proposition states that the function $g^*_i$ constructed in the proof of Proposition \ref{prop:optimality_weak_convergence_dual} is a Borel measurable version of the $c$-transform of $f^*_i$, where $f_i^*$ was defined in \eqref{eq:Limsupf_i}.

\begin{proposition}\label{prop : meaurable version}
Let $\{g^n\}_{n \in \N}$ and $\{ f_n \}_{n \in \N}$ be as in the proof of Proposition \ref{prop:optimality_weak_convergence_dual}, let $g^*$ be the weak$^*$ limit of the $g^n$, and let $f^*$ be as defined in \eqref{eq:Limsupf_i}. Then, for every $i\in \Y$,
\begin{equation}\label{eq: meaurable version g_i}
    g_i^*(x) = \inf_{\widetilde{x} \in \mathcal{X}} \left\{ f^*_i(\widetilde{x}) + c(x, \widetilde{x}) \right\}
\end{equation}
for $\mu_i$-a.e. $x\in \X$. This statement must be interpreted as: the set in which \eqref{eq: meaurable version g_i} is violated is contained in a Borel measurable set with zero $\mu_i$ measure.
\end{proposition}

\begin{proof} 
From the proof of Theorem \ref{thm : measurable robust classifier} it holds that for each $i \in \Y$
\begin{equation}
    \int_{\mathcal{X}} f^*_i(\widetilde{x}) d\widetilde{\mu}^*_i(\widetilde{x}) + \int_{\mathcal{X} \times \mathcal{X}} c(x, \widetilde{x}) d\pi^*_i(x, \widetilde{x}) = \int_{\mathcal{X}} g_i^*(x) d\mu_i(x).
    \label{eqn:AUXProp43}
\end{equation}
On the other hand, from the definition of $f_i^n$ it follows that
\[  g_i^n(x) \leq f_i^n(\tilde x) +c_n(x, \tilde x), \quad \forall \tilde x \in \X,  \text{ and } \mu_i\text{-a.e. } x\in \X.  \]
We can then combine the above with Lemma \ref{lem:LimsupWeak^*} to conclude that for $\mu_i$-a.e. $x\in \X$ and every $\tilde x \in \X$ we have
\begin{align*}
   g_i^*(x) \leq \limsup_{n \rightarrow \infty} g_i^n(x) \leq  \limsup_{n \rightarrow \infty}  f_i^n(\tilde x ) + c_n(x, \tilde x)  & = f_i^*(\tilde x) + c(x, \tilde x).
\end{align*}
Taking the inf over $\tilde x \in \X$ we conclude that for $\mu_i$-a.e. $x\in \X$ we have
\begin{equation}\label{eq:AuxgI*}
    g_i^*(x)\leq \inf_{\tilde x \in \X}  \{ f_i^* (\tilde x) + c(x, \tilde x)  \}. 
\end{equation}
From this and \eqref{eqn:AUXProp43} we see that $g_i^* \in L^1(\mu_i)$ and $-f_i^* \in L^1(\tilde {\mu}_i)$ are optimal dual potentials for the optimal transport problem $C(\mu_i, \widetilde \mu^*_i)$. If \eqref{eq:AuxgI*} did not hold with equality for $\mu_i$-a.e. $x \in \X$, then we would be able to construct a Borel-measurable version $h_i$ of the right hand side of \eqref{eq:AuxgI*} (see Lemma \ref{lemma : Universally measurable functions} in the Appendix) which would be strictly greater than $g_i^*$ in a set of positive $\mu_i$-measure. In addition, we would have that $(h_i, -f_i^*)$ is a feasible dual pair for the OT problem $C(\mu_i, \tilde{\mu}_i)$. However, the above would contradict the optimality of the dual potentials $(g_i^*, -f_i^*)$. We thus conclude that \eqref{eq:AuxgI*} holds with equality except on a set contained in a set of $\mu_i$ measure zero.  
\end{proof}

\subsection{Well-posedness of the closed-ball model \eqref{def:closed-ball_model-corrected}}
\label{sec:ClosedBallWell}

\begin{proof}[Proof of Proposition \ref{prop : borel_general_closed_ball}]
We actually prove that for arbitrary cost $c$ satisfying Assumption \ref{assump}, the solution $f^*$ to \eqref{def:DRO_model} constructed in the proof of Theorem \ref{thm : measurable robust classifier} is also a solution for the problem:
\begin{equation}\label{def: borel_general_closed_ball}
    \inf_{f \in \mathcal{F}} \left\{ \sum_{i \in \Y} \int_{\mathcal{X}} \sup_{\widetilde{x} \in \mathcal{X}} \left\{ 1 - f_i(\widetilde{x}) - c(x, \widetilde{x}) \right\} d\overline{\mu}_i(x) \right\}.
\end{equation}
Proposition \ref{prop : borel_general_closed_ball} will then be an immediate consequence of this more general result when applied to $c=c_\varepsilon$.

Let $f^*$ be the Borel solution of \eqref{def:DRO_model} constructed in the proof of Theorem \ref{thm : measurable robust classifier}. It suffices to show that for any $f \in \mathcal{F}$
\begin{equation*}
    \sum_{i \in \Y} \int_{\mathcal{X}} \inf_{\widetilde{x} \in \mathcal{X}} \left\{ f^*_i(\widetilde{x}) + c(x, \widetilde{x}) \right\} d\overline{\mu}_i(x) \geq \sum_{i \in \Y} \int_{\mathcal{X}} \inf_{\widetilde{x} \in \mathcal{X}} \left\{ f_i(\widetilde{x}) + c(x, \widetilde{x}) \right\} d\overline{\mu}_i(x).
\end{equation*}
Suppose not. Then there exists some $\widehat{f} \in \mathcal{F}$ which provides a strict inequality in the opposite direction. Now, on one hand, \eqref{eq: meaurable version g_i} of Proposition \ref{prop : meaurable version} implies
\begin{equation*}
     \sum_{i \in \Y} \int_{\mathcal{X}} \inf_{\widetilde{x} \in \mathcal{X}} \left\{ f^*_i(\widetilde{x}) + c(x, \widetilde{x}) \right\} d\overline{\mu}_i(x) =  \sum_{i \in \Y} \int_{\mathcal{X}} g_i^*(x) d\mu_i(x).
\end{equation*}
On the other hand, by Lemma \ref{lemma : Universally measurable functions}, for each $i \in \Y$ there exists a Borel measurable function $\widehat{g}_i$ equal to $\inf_{\widetilde{x} \in \mathcal{X}} \{ \widehat{f}_i(\widetilde{x}) + c(x, \widetilde{x}) \}$ $\overline{\mu}_i$-almost everywhere. Let $\widehat{g}:=(\widehat{g}_1, \dots, \widehat{g}_K)$. Combining the existence of such $\widehat{g}$ with the above equation, and using \eqref{eq : Universally measurable functions}, it follows that $\widehat{g}$ satisfies
\begin{equation}\label{eq : contradictory inequality}
    \sum_{i \in \Y} \int_{\mathcal{X}} g_i^*(x) d\mu_i(x) < \sum_{i \in \Y} \int_{\mathcal{X}} \widehat{g}_i(x) d\mu_i(x).
\end{equation}
Notice that for each $A \in S_K$ and $\otimes \overline{\mu}_i$-almost everywhere $x_1, \dots x_K$, we have 
\[ \sum_{i \in A}\inf_{\widetilde{x} \in \mathcal{X}} \left\{ \widehat{f}_i(\widetilde{x}) + c(x_i, \widetilde{x}) \right\} \leq \inf_{\widetilde{x} \in \mathcal{X}} \left\{ \sum_{i \in A} \widehat{f}_i(\widetilde{x}) + c(x_i, \widetilde{x}) \right\}  \leq 1 + c_A(x_1, \dots, x_K).  \] 
From the above we conclude that $\widehat{g}$ is feasible for \eqref{eq:mot_decomposed_dual_Linfty}. However, this and \eqref{eq : contradictory inequality} combined contradict the fact that $g^*$ is optimal for \eqref{eq:mot_decomposed_dual_Linfty}, as had been shown in Proposition \ref{prop:optimality_weak_convergence_dual}.
\end{proof}

\begin{proof}[Proof of Corollary \ref{cor:BinaryCase}]
It is straightforward to verify (e.g., see \cite{bungert2023geometry}) that for $(f_1,1-f_1) \in \F$ we can write 
\begin{equation}
    \overline{R}_\varepsilon((f_1, 1-f_1)) = \int_0^1 \overline{R}_\varepsilon((\mathds{1}_{\{f_1 \geq t \}}, \mathds{1}_{ \{f_1 \geq t \}^c})) dt.
    \label{eq:Coarea}
\end{equation}
It is also straightforward to see that
\[  \overline{R}_\varepsilon((\mathds{1}_{\{f_1 \geq t \}}, \mathds{1}_{ \{f_1 \geq t \}^c})) = \int_{\X} \sup_{ \tilde x \in \overline{B}_\varepsilon(x)}  \mathds{1}_{A^c}(\tilde x) d\overline{\mu}_1(x) + \int_{\X} \sup_{ \tilde x \in \overline{B}_\varepsilon(x)}  \mathds{1}_{A}(\tilde x) d\overline{\mu}_2(x).  \]

Let $(f_1, 1-f_1)$ be a solution to \eqref{def:closed-ball_model-corrected} (which by remark \ref{rem:01Loss} can indeed be taken of this form). It follows from \eqref{eq:Coarea} that for almost every $t \in [0,1]$ the pair $(\mathds{1}_{\{f_1 \geq t \}}, \mathds{1}_{ \{f_1 \geq t \}^c})$ is also a solution for that same problem and thus also for the problem restricted to hard-classifiers. This proves the desired result.   
\end{proof}

\subsection{Connection between closed-ball model and open-ball model}\label{section: unify all models}
\begin{proof}[Proof of Theorem \ref{thm:AllmodelsSame}]

One can easily observe that for any fixed $\varepsilon > 0$ and $\delta > 0$ we have
\begin{equation*}
    \sup_{\widetilde{x} \in B_{\varepsilon}(x)} \{ 1 - f_i(\widetilde{x}) \}  \leq \sup_{\widetilde{x} \in \overline{B}_{\varepsilon}(x)} \{ 1 - f_i(\widetilde{x}) \} \leq \sup_{\widetilde{x} \in B_{\varepsilon + \delta}(x)} \{ 1 - f_i(\widetilde{x}) \} 
\end{equation*}
for all $x \in \mathcal{X}$ and all $f\in \F$. This simple observation leads to $R^o_{\varepsilon}(f) \leq \overline{R}_{\varepsilon}(f) \leq R^o_{\varepsilon + \delta}(f)$ for all $f \in \F$.
Thus we also have $R_\varepsilon^o \leq \overline{R}_\varepsilon \leq  R_{\varepsilon +\delta}^o  $ and in particular $R_\varepsilon^o \leq \overline{R}_\varepsilon \leq \liminf_{\delta \rightarrow 0} R_{\varepsilon + \delta}^o $.  From the above we can also see that the function $\varepsilon \mapsto R_{\varepsilon}^o $ is non-decreasing and as such it is continuous for all but at most countably many values of $\varepsilon>0$. Therefore, for all but at most countably many $\varepsilon$ we have $R_\varepsilon^o = \overline{R}_\varepsilon $. 

Now, let $f^*$ be  solution of \eqref{def:closed-ball_model-corrected} and assume we have $R_\varepsilon^o = \overline{R}_\varepsilon $. Then
\[ 
R_\varepsilon^o(f^*) \leq \overline{R}_\varepsilon(f^*) = \overline{R}_\varepsilon = R_\varepsilon^o,  
\]
which means $f^*$ is a solution of \eqref{def:open-ball_model}.
\end{proof}

\section{Conclusion and future works}\label{section : conclusion and future works}

In this paper, we establish the equivalence of three popular models of multiclass adversarial training, the open ball model, the closed ball model, and the distributionally robust optimization model and for the first time (with the exception of partial results in \cite{bungert2023geometry}) prove the existence of Borel measurable optimal robust classifiers in the agnostic-classifier setting.  We are able to unify these models via a framework we have developed that connects these problems to optimal transport and total variation minimization problems.  Notably, our results show that it is unnecessary to grapple with the cumbersome machinery of universal sigma algebras, which was needed to prove existence of classifiers in past results.


Although our analysis sheds light on this area, many open questions still remain on both the theoretical and practical side. One of the most important practical questions is how to extend these results when the set of classifiers $\mathcal{F}$ is some parametric family, for example neural networks.  In particular, one would like to specify the properties a parametric family must satisfy in order to approximate robust classifiers to some desired degree of accuracy.  In the case of neural networks, one might ask for the number of neurons or number of layers that are required for robust classification. 

Related to the above practical question is the following geometric/theoretical question:
given an optimal robust classifier $f^*$, can we give a characterization of the regularity of $f^*$ as in \cite{bungert2023geometry}?  In particular, one would like to quantify the smoothness of the interfaces between the different classes. In general, we cannot guarantee that $f^*$ is a hard classifier, thus, this problem is best posed as a question about the smoothness of the level sets of $f^*$. 
Since optimal classifiers need not be unique, one can also pose the more general question of when it is possible to find at least one  optimal Borel robust classifier with some specified regularity property.
Due to the connection between approximation and regularity, answering this question will provide insights to the previous question of how well one can approximate optimal robust classifiers using certain parametric families.

A final question is how to extend our framework to other more general settings. In this paper, we have assumed throughout that the loss function is the $0$-$1$ loss.  However, most practitioners prefer strongly convex loss functions, for example, the cross entropy function, which allows for faster optimization and has other desirable properties.  As a result, one would like to establish the analog of these results in this more general setting.  This would be crucial for bringing these theoretical insights closer to the models favored by working practitioners.


\section*{Acknowledgments}
NGT is supported by the NSF grants DMS-2005797 and DMS-2236447.

\bibliographystyle{siamplain}
\bibliography{references}

\begin{thebibliography}{10}

\bibitem{AguehCarlier}
{\sc M.~Agueh and G.~Carlier}, {\em Barycenters in the wasserstein space}, SIAM
  Journal on Mathematical Analysis, 43 (2011), pp.~904--924,
  \url{https://doi.org/10.1137/100805741},
  \url{https://doi.org/10.1137/100805741},
  \url{https://arxiv.org/abs/https://doi.org/10.1137/100805741}.

\bibitem{awasthi2021existence}
{\sc P.~Awasthi, N.~Frank, and M.~Mohri}, {\em On the existence of the
  adversarial bayes classifier}, Advances in Neural Information Processing
  Systems, 34 (2021), pp.~2978--2990.

\bibitem{awasthi2021extended}
{\sc P.~Awasthi, N.~S. Frank, and M.~Mohri}, {\em On the existence of the
  adversarial bayes classifier (extended version)}, arXiv preprint
  arXiv:2112.01694,  (2021).

\bibitem{pmlr-v206-balcan23a}
{\sc M.-F. Balcan, R.~Pukdee, P.~Ravikumar, and H.~Zhang}, {\em Nash equilibria
  and pitfalls of adversarial training in adversarial robustness games}, in
  Proceedings of The 26th International Conference on Artificial Intelligence
  and Statistics, F.~Ruiz, J.~Dy, and J.-W. van~de Meent, eds., vol.~206 of
  Proceedings of Machine Learning Research, PMLR, 25--27 Apr 2023,
  pp.~9607--9636, \url{https://proceedings.mlr.press/v206/balcan23a.html}.

\bibitem{doi:10.1137/16M1070426}
{\sc A.~L. Bertozzi and A.~Flenner}, {\em Diffuse interface models on graphs
  for classification of high dimensional data}, SIAM Review, 58 (2016),
  pp.~293--328, \url{https://doi.org/10.1137/16M1070426}.

\bibitem{bertsekas1996stochastic}
{\sc D.~Bertsekas and S.~E. Shreve}, {\em Stochastic optimal control: the
  discrete-time case}, vol.~5, Athena Scientific, 1996.

\bibitem{Bhagoji2019LowerBO}
{\sc A.~Bhagoji, D.~Cullina, and P.~Mittal}, {\em Lower bounds on adversarial
  robustness from optimal transport}, in NeurIPS, 2019.

\bibitem{BIGGIO2018317}
{\sc B.~Biggio and F.~Roli}, {\em Wild patterns: Ten years after the rise of
  adversarial machine learning}, Pattern Recognition, 84 (2018), pp.~317--331,
  \url{https://doi.org/https://doi.org/10.1016/j.patcog.2018.07.023},
  \url{https://www.sciencedirect.com/science/article/pii/S0031320318302565}.

\bibitem{MR3959085}
{\sc J.~Blanchet and K.~Murthy}, {\em Quantifying distributional model risk via
  optimal transport}, Math. Oper. Res., 44 (2019), pp.~565--600,
  \url{https://doi.org/10.1287/moor.2018.0936},
  \url{https://doi-org.ezproxy.library.wisc.edu/10.1287/moor.2018.0936}.

\bibitem{bose2020adversarial}
{\sc J.~Bose, G.~Gidel, H.~Berard, A.~Cianflone, P.~Vincent, S.~Lacoste-Julien,
  and W.~Hamilton}, {\em Adversarial example games}, Advances in neural
  information processing systems, 33 (2020), pp.~8921--8934.

\bibitem{Bousquet2004}
{\sc O.~Bousquet, S.~Boucheron, and G.~Lugosi}, {\em Introduction to
  Statistical Learning Theory}, Springer Berlin Heidelberg, Berlin, Heidelberg,
  2004, pp.~169--207, \url{https://doi.org/10.1007/978-3-540-28650-9_8},
  \url{https://doi.org/10.1007/978-3-540-28650-9_8}.

\bibitem{MR4163472}
{\sc Z.~M. Boyd, M.~A. Porter, and A.~L. Bertozzi}, {\em Stochastic block
  models are a discrete surface tension}, J. Nonlinear Sci., 30 (2020),
  pp.~2429--2462, \url{https://doi.org/10.1007/s00332-019-09541-8},
  \url{https://doi-org.ezproxy.library.wisc.edu/10.1007/s00332-019-09541-8}.

\bibitem{bungert2023geometry}
{\sc L.~Bungert, N.~Garc{\'\i}a~Trillos, and R.~Murray}, {\em The geometry of
  adversarial training in binary classification}, Information and Inference: A
  Journal of the IMA, 12 (2023), pp.~921--968.

\bibitem{bungert2022gamma}
{\sc L.~Bungert and K.~Stinson}, {\em Gamma-convergence of a nonlocal perimeter
  arising in adversarial machine learning}, arXiv preprint arXiv:2211.15223,
  (2022).

\bibitem{caroccia2020mumford}
{\sc M.~Caroccia, A.~Chambolle, and D.~Slep{\v{c}}ev}, {\em Mumford--shah
  functionals on graphs and their asymptotics}, Nonlinearity, 33 (2020),
  p.~3846.

\bibitem{MR4106615}
{\sc A.~Cristofari, F.~Rinaldi, and F.~Tudisco}, {\em Total variation based
  community detection using a nonlinear optimization approach}, SIAM J. Appl.
  Math., 80 (2020), pp.~1392--1419, \url{https://doi.org/10.1137/19M1270446},
  \url{https://doi-org.ezproxy.library.wisc.edu/10.1137/19M1270446}.

\bibitem{frank2022existence}
{\sc N.~S. Frank}, {\em Existence and minimax theorems for adversarial
  surrogate risks in binary classification}, arXiv preprint arXiv:2206.09098,
  (2022).

\bibitem{frank2022consistency}
{\sc N.~S. Frank and J.~Niles-Weed}, {\em The consistency of adversarial
  training for binary classification}, arXiv preprint arXiv:2206.09099,
  (2022).

\bibitem{trillos2023multimarginal}
{\sc N.~Garc{\i}a~Trillos, M.~Jacobs, and J.~Kim}, {\em The multimarginal
  optimal transport formulation of adversarial multiclass classification},
  Journal of Machine Learning Research, 24 (2023), pp.~1--56.

\bibitem{NGTJakwangMattArxiv}
{\sc N.~{Garc\'ia Trillos}, M.~Jacobs, and J.~Kim}, {\em The multimarginal
  optimal transport formulation of adversarial multiclass classification},
  arXiv preprint https://arxiv.org/abs/2204.12676,  (2023).

\bibitem{GTMurray2017}
{\sc N.~{Garc\'ia Trillos} and R.~Murray}, {\em A new analytical approach to
  consistency and overfitting in regularized empirical risk minimization},
  European Journal of Applied Mathematics, 28 (2017), p.~886–921,
  \url{https://doi.org/10.1017/S0956792517000201}.

\bibitem{Trillos2020AdversarialCN}
{\sc N.~Garc{\'\i}a~Trillos and R.~Murray}, {\em Adversarial classification:
  Necessary conditions and geometric flows}, Journal of Machine Learning
  Research, 23 (2022), pp.~1--38.

\bibitem{MR4419604}
{\sc N.~Garc\'{\i}a~Trillos, R.~Murray, and M.~Thorpe}, {\em From graph cuts to
  isoperimetric inequalities: convergence rates of {C}heeger cuts on data
  clouds}, Arch. Ration. Mech. Anal., 244 (2022), pp.~541--598,
  \url{https://doi.org/10.1007/s00205-022-01770-8},
  \url{https://doi-org.ezproxy.library.wisc.edu/10.1007/s00205-022-01770-8}.

\bibitem{goodfellow_examples}
{\sc I.~Goodfellow, J.~Shlens, and C.~Szegedy}, {\em Explaining and harnessing
  adversarial examples}, in International Conference on Learning
  Representations, 2015, \url{http://arxiv.org/abs/1412.6572}.

\bibitem{BertozziHuLaurent}
{\sc H.~Hu, T.~Laurent, M.~A. Porter, and A.~L. Bertozzi}, {\em A method based
  on total variation for network modularity optimization using the {MBO}
  scheme}, SIAM J. Appl. Math., 73 (2013), pp.~2224--2246,
  \url{https://doi.org/10.1137/130917387},
  \url{https://doi-org.ezproxy.library.wisc.edu/10.1137/130917387}.

\bibitem{luo2017convergence}
{\sc X.~Luo and A.~L. Bertozzi}, {\em Convergence of the graph allen--cahn
  scheme}, Journal of Statistical Physics, 167 (2017), pp.~934--958.

\bibitem{luzin1919quelques}
{\sc N.~N. Luzin and W.~Sierpi{\'n}ski}, {\em Sur quelques propri{\'e}t{\'e}s
  des ensembles (a)},  (1919).

\bibitem{MBO_Scheme_on_Graphs}
{\sc E.~Merkurjev, T.~Kosti\'{c}, and A.~L. Bertozzi}, {\em An mbo scheme on
  graphs for classification and image processing}, SIAM Journal on Imaging
  Sciences, 6 (2013), pp.~1903--1930, \url{https://doi.org/10.1137/120886935}.

\bibitem{Meunier2021MixedNE}
{\sc L.~Meunier, M.~Scetbon, R.~B. Pinot, J.~Atif, and Y.~Chevaleyre}, {\em
  Mixed nash equilibria in the adversarial examples game}, in Proceedings of
  the 38th International Conference on Machine Learning, M.~Meila and T.~Zhang,
  eds., vol.~139 of Proceedings of Machine Learning Research, PMLR, 18--24 Jul
  2021, pp.~7677--7687,
  \url{https://proceedings.mlr.press/v139/meunier21a.html}.

\bibitem{Nishiura}
{\sc T.~Nishiura}, {\em Absolute measurable spaces}, vol.~120 of Encyclopedia
  of Mathematics and its Applications, Cambridge University Press, Cambridge,
  2008, \url{https://doi.org/10.1017/CBO9780511721380},
  \url{https://doi-org.ezproxy.library.wisc.edu/10.1017/CBO9780511721380}.

\bibitem{Pydi2021AdversarialRV}
{\sc M.~S. Pydi and V.~Jog}, {\em Adversarial risk via optimal transport and
  optimal couplings}, IEEE Transactions on Information Theory, 67 (2021),
  pp.~6031--6052.

\bibitem{pydi2021the}
{\sc M.~S. Pydi and V.~Jog}, {\em The many faces of adversarial risk}, in
  Advances in Neural Information Processing Systems, M.~Ranzato,
  A.~Beygelzimer, Y.~Dauphin, P.~Liang, and J.~W. Vaughan, eds., vol.~34,
  Curran Associates, Inc., 2021, pp.~10000--10012,
  \url{https://proceedings.neurips.cc/paper_files/paper/2021/file/52c4608c2f126708211b9e0a60eaf050-Paper.pdf}.

\bibitem{MR3207626}
{\sc Y.~van Gennip, N.~Guillen, B.~Osting, and A.~L. Bertozzi}, {\em Mean
  curvature, threshold dynamics, and phase field theory on finite graphs},
  Milan J. Math., 82 (2014), pp.~3--65,
  \url{https://doi.org/10.1007/s00032-014-0216-8},
  \url{https://doi-org.ezproxy.library.wisc.edu/10.1007/s00032-014-0216-8}.

\bibitem{MR1964483}
{\sc C.~Villani}, {\em Topics in optimal transportation}, vol.~58 of Graduate
  Studies in Mathematics, American Mathematical Society, Providence, RI, 2003,
  \url{https://doi.org/10.1090/gsm/058},
  \url{https://doi-org.ezproxy.library.wisc.edu/10.1090/gsm/058}.

\bibitem{Oldandnew}
{\sc C.~Villani}, {\em Optimal transport}, vol.~338 of Grundlehren der
  mathematischen Wissenschaften [Fundamental Principles of Mathematical
  Sciences], Springer-Verlag, Berlin, 2009,
  \url{https://doi.org/10.1007/978-3-540-71050-9},
  \url{https://doi-org.ezproxy.library.wisc.edu/10.1007/978-3-540-71050-9}.
\newblock Old and new.

\bibitem{LUXBURG2011651}
{\sc U.~von Luxburg and B.~Schölkopf}, {\em Statistical learning theory:
  Models, concepts, and results}, in Inductive Logic, D.~M. Gabbay,
  S.~Hartmann, and J.~Woods, eds., vol.~10 of Handbook of the History of Logic,
  North-Holland, 2011, pp.~651--706,
  \url{https://doi.org/https://doi.org/10.1016/B978-0-444-52936-7.50016-1},
  \url{https://www.sciencedirect.com/science/article/pii/B9780444529367500161}.

\end{thebibliography}

\section*{Appendix}\label{section : Appendix}

\subsection*{Weak$^*$ topology}
\label{sec:WeakStarTopology}
\begin{definition}[Weak$^*$ topology]\label{definition : weak* topology}
Let $\mu=(\mu_1, \dots, \mu_K) \in \prod_{i=1}^K \mathcal{M}_+(\mathcal{X})$. For a sequence $\{h^n\}_{n \in \N} \subseteq \prod_{i \in \Y} L^{\infty}(\mathcal{X};  \mu_i)$, we say that $\{h^n\}$ weak$^*$-converges to $h \in \prod_{i \in \Y} L^{\infty}(\mathcal{X};  \mu_i)$ if for any $q \in \prod_{i \in \Y} L^{1}(\mathcal{X};  \mu_i)$, it holds that
\begin{equation}\label{def:weak*_convergence}
    \lim_{n \to \infty} \int_{\mathcal{X}} h_i^n(x) q_i(x) d\mu_i(x) = \int_{\mathcal{X}} h_i(x) q_i(x) d\mu_i(x)
\end{equation}
for all $i \in \Y$.
\end{definition}

\begin{remark}
Note that for a Borel positive measure $\rho$ which is either finite or $\sigma$-finite over a Polish space, the dual of $L^1(\rho)$ is $L^{\infty}(\rho)$, which justifies the definition \eqref{def:weak*_convergence}.
\end{remark}

\begin{lemma}
\label{lem:LimsupWeak^*}
Suppose $\{ g_i^n \}_{n \in \N}$ is a sequence of measurable real-valued functions over $\X$ satisfying $0 \leq g_i^n \leq 1$ for every $n \in \N$. Suppose that $g_i^n $ converges in the weak$^*$ topology of $L^\infty(\X; \mu_i)$ toward $g_i$, where $\mu_i$ is a finite positive measure. Then, for $\mu_i$-a.e. $x \in X$, we have
\[\limsup_{n \rightarrow \infty} g_i^n (x) \geq g_i(x).\]
\end{lemma}
\begin{proof}
Let $E$ be a measurable subset of $\X$. Then 
\[ \int_{\X} (\limsup_{n \rightarrow \infty} g^n_i(x) - g_i(x) ) \mathds{1}_E(x) d\mu_i(x) \geq  \limsup_{n \rightarrow \infty} \int_{\X} ( g^n_i(x) - g_i(x) ) \mathds{1}_E(x) d\mu_i(x) = 0 ,  \]
by the reverse Fatou inequality and the assumption that the sequence $\{g_i^n\}_{n \in \N}$ converges in the weak$^*$ sense toward $g_i$. Since $E$ was arbitrary, the result follows.    
\end{proof}

\subsection*{$c$-transform}\label{appendix : c-transform}
$c$-transform has an important role in optimal transport theory. One can characterize an optimizer of a dual problem by iterating $c$-transform: see \cite{MR1964483, Oldandnew} for more details.

\begin{definition}[$c$-transform in \cite{Oldandnew}]
Let $\mathcal{X}, \mathcal{X}'$ be measurable spaces, and let $c : \mathcal{X} \times \mathcal{X}' \to (-\infty, \infty]$. Given a measurable function $h : \mathcal{X} \to \mathbb{R} \cup \{\infty, -\infty\}$, its $c$-transform is defined as
\begin{equation*}
    h^{c}(x'):=\inf_{x \in \mathcal{X}} \{ h(x) + c(x, x') \}.
\end{equation*}
Similarly, for $g : \mathcal{X}' \to \mathbb{R} \cup \{\infty, -\infty\}$, its $\overline{c}$-transform is defined as
\begin{equation*}
    g^{\overline{c}}(x):=\sup_{x' \in \mathcal{X}'} \{ g(x') - c(x, x') \}.
\end{equation*}
\label{def:cTransforms}
\end{definition}

\begin{proposition}
For any mearurable functions $h$ over $\mathcal{X}$ and $g$ over $\mathcal{X}'$, and cost function $c : \mathcal{X} \times \mathcal{X}' \to (-\infty, \infty]$, it holds that for every $(x, x') \in \mathcal{X} \times \mathcal{X}'$,
\begin{equation*}
    h^{c}(x') - h(x) \leq c(x, x'), \quad g(x') - g^{\overline{c}}(x) \leq c(x,x').
\end{equation*}
\end{proposition}

\begin{theorem}[Theorem 5.10 in \cite{Oldandnew}]\label{thm : duality and c-transform}
Let $\mathcal{X}$ be a Polish space and $c(\cdot, \cdot)$ be a cost function bounded from below and lower semi-continuous. Then, for $\nu, \widetilde{\nu} \in \mathcal{P}(\mathcal{X})$,
\begin{align*}
    \inf_{\pi_i \in \Gamma(\nu, \widetilde{\nu})}\int_{\mathcal{X} \times \mathcal{X}} c(x, \widetilde{x}) d\pi_i(x, \widetilde{x}) &= \sup_{g_i, f_i \in \mathcal{C}_b, g_i-f_i \leq c} \left\{ \int_{\mathcal{X}} g_i(x) d\nu(x) - \int_{\mathcal{X}} f_i(\widetilde{x}) d\widetilde{\nu}(\widetilde{x}) \right\}\\
    &= \sup_{f_i \in L^1(\widetilde{\nu})} \left\{ \int_{\mathcal{X}} (f_i)^c(x) d\nu(x) - \int_{\mathcal{X}} f_i(\widetilde{x}) d\widetilde{\nu}(\widetilde{x}) \right\}\\
    &= \sup_{g_i \in L^1(\nu)} \left\{ \int_{\mathcal{X}} g_i(x) d\nu(x) - \int_{\mathcal{X}} (g_i)^{\overline{c}}(\widetilde{x}) d\widetilde{\nu}(\widetilde{x}) \right\}.
\end{align*}
Furthermore, the infimum is indeed a minimum. However, the supremum may not be achieved.
\end{theorem}

\subsection*{Decomposition of universally measurable functions}\label{subsection : Universally measurable functions}

\begin{lemma}\label{lemma : Universally measurable functions}
Let $\mathcal{X}$ be a Polish space, $\mu$ and let $\overline{\mu}$ be a Borel probability measure and its extension to the universal $\sigma$-algebra, respectively. Let $f$ be a universally measurable function for which $\int_{\X} |f(x)| d \overline{\mu}(x) <\infty$. Then there exists a Borel measurable function $g$ such that $f=g$ $\overline{\mu}$-almost everywhere. Also,
\begin{equation}\label{eq : Universally measurable functions}
    \int_{\mathcal{X}} f(x) d\overline{\mu}(x) = \int_{\mathcal{X}} g(x) d\mu(x).
\end{equation}
\end{lemma}

\begin{proof}
Without the loss of generality we can assume that $f \geq 0$. Since $f$ is universally measurable, we can write 
\begin{equation*}
    f(x) = \lim_{n \to \infty } f_n(x) := \lim_{n \to \infty } \sum_{k=1}^n c_k^n \mathds{1}_{A_k^n}(x),
\end{equation*}
for positive coefficients $c_1^n, \dots, c_n^n$ and $A_1^n , \dots, A_n^n$ universally measurable and pairwise disjoint sets. By the definition of universally measurable sets, for each $A_k^n$ there exists a Borel set $B_k^n$ such that $\overline{\mu}(A_k^n \setminus B_k^n) = 0$. Hence, for each $n \in \mathbb{N}$, we can write
\begin{equation*}
    f_n(x) = \sum_{k=1}^n c_k^n \mathds{1}_{B_k^n}(x) + \sum_{k=1}^n c_k^n \mathds{1}_{C_k^n}(x),
\end{equation*}
where $C_k^n = A_k^n \setminus B_k^n$.   We conclude that
\begin{equation*}
    f(x) = g(x) + h(x) := \limsup_{n \to \infty } \sum_{k=1}^n c_k^n \mathds{1}_{B_k^n}(x) + \liminf_{n \to \infty } \sum_{k=1}^n c_k^n \mathds{1}_{C_k^n}(x)
\end{equation*}
where $g$ is Borel measurable, $h$ is universally measurable and $h=0$ $\overline{\mu}$-almost everywhere. 

Since $f=g$ $\overline{\mu}$-almost everywhere and $g$ is Borel measurable, then 
\begin{equation*}
    \int_{\mathcal{X}} f(x) d\overline{\mu}(x) = \int_{\mathcal{X}} g(x) d\overline{\mu}(x) = \int_{\mathcal{X}} g(x) d\mu(x),
\end{equation*}
from which \eqref{eq : Universally measurable functions} follows.
\end{proof}

\end{document}